\documentclass{article}

\usepackage[preprint]{neurips_2023}

\usepackage[utf8]{inputenc} 
\usepackage[T1]{fontenc}    
\usepackage{hyperref}       
\usepackage{url}            
\usepackage{booktabs}       
\usepackage{nicefrac}       
\usepackage{microtype}      
\usepackage{xcolor}         
\usepackage{graphicx}
\usepackage{xspace}
\usepackage{float}
\usepackage{rotating}
\usepackage{amsmath,amsfonts,amssymb,amsthm}
\usepackage{comment}
\usepackage{mdframed,lipsum}
\newtheorem{theorem}{Theorem}[section]
\newtheorem{proposition}{Proposition}[section]
\newtheorem{corollary}{Corollary}[section]
\newtheorem{lemma}{Lemma}[section]

\newcommand{\bg}{\boldsymbol{\gamma}}
\newcommand{\h}{\mathbf{h}}
\newcommand{\V}{\mathbf{v}}
\newcommand{\w}{\mathbf{w}}
\newcommand{\x}{\mathbf{x}}
\newcommand{\NC}{\ensuremath{\mathcal{NC}}\xspace}

\usepackage{amsmath,amsfonts,bm}

\def\eqref#1{equation~\ref{#1}}

\def\1{\bm{1}}

\def\rvh{{\mathbf{h}}}
\def\rvu{{\mathbf{i}}}

\def\rvu{{\mathbf{u}}}
\def\rvv{{\mathbf{v}}}

\DeclareMathAlphabet{\mathsfit}{\encodingdefault}{\sfdefault}{m}{sl}
\SetMathAlphabet{\mathsfit}{bold}{\encodingdefault}{\sfdefault}{bx}{n}

\newcommand{\R}{\mathbb{R}}

\newif\ifdoublecol
\doublecolfalse

\title{Towards Understanding Neural Collapse: The Effects of Batch Normalization and Weight Decay}

\author{%
  Leyan Pan \\
  Georgia Institute of Technology\\
  Atlanta, GA, 30332 \\
  \texttt{leyanpan@gatech.edu} \\
  \And
    Xinyuan Cao \\
  Georgia Institute of Technology\\
  Atlanta, GA, 30332 \\
  \texttt{xcao78@gatech.edu} \\
}

\begin{document}

\maketitle

\begin{abstract}

Neural Collapse (\NC) is a geometric structure recently observed at the terminal phase of training deep neural networks, which states that last-layer feature vectors for the same class would "collapse" to a single point, while features of different classes become equally separated. We demonstrate that batch normalization (BN) and weight decay (WD) critically influence the emergence of \NC. In the near-optimal loss regime, we establish an asymptotic lower bound on the emergence of \NC that depends only on the WD value, training loss, and the presence of last-layer BN. Our experiments substantiate theoretical insights by showing that models demonstrate a stronger presence of \NC with BN, appropriate WD values, lower loss, and lower last-layer feature norm. Our findings offer a novel perspective in studying the role of BN and WD in shaping neural network features.
\end{abstract}
\vspace{-0.02in}
\section{Introduction}
\vspace{-0.02in}
The wide application of deep learning models has raised significant interest in theoretically understanding the mechanisms underlying their success. In particular, the generalization capability of overparameterized networks continues to escape the grasp of traditional learning theory, and the quantitative roles and impacts of widely adapted training techniques including batch normalization (\textbf{BN}, \citet{bn}) and weight decay (\textbf{WD}, \citet{loshchilov2017decoupled}) remains an area of active investigation. 

A promising way of mechanistically understanding neural networks is by analyzing their feature learning process. \citet{doi:10.1073/pnas.2015509117} observed an elegant mathematical structure in well-trained neural network classifiers, termed ``Neural Collapse" (abbreviated \NC in this work,  see Figure \ref{fig:ncillu} for detailed visualization.) \NC states that after sufficient training of the neural networks: {\bf NC1 }\textit{(Variability Collapse)}: The intra-class variability of the last-layer feature vectors tends to be zero; {\bf NC2 }\textit{(Convergence to Simplex ETF)}: The mean of the class feature vectors become equal-norm and form a Simplex Equiangular Tight Frame (ETF) around the center up to re-scaling; {\bf NC3 }{\it (Self-Duality)}: The last layer weights converge to match the class mean features up to re-scaling; {\bf NC4 }{\it (Convergence to NCC)}: The last layer of the network behaves the same as ``Nearest Class Center".

\begin{figure*}[h!]
\centering
\includegraphics[width=\textwidth]{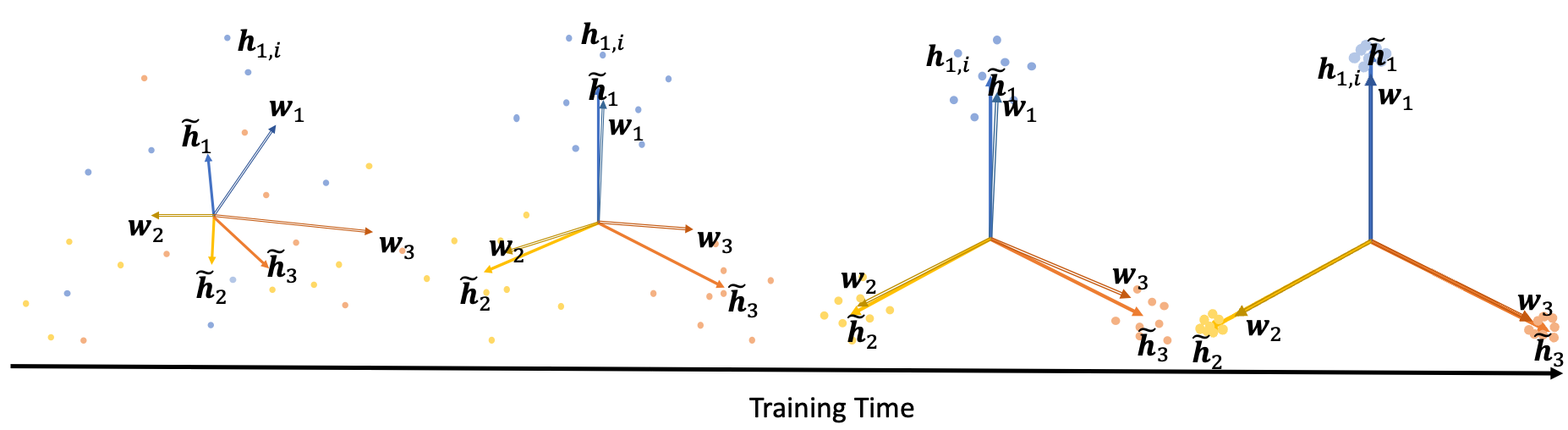}
\vspace{-0.25in}
\caption{Visualization of \NC(\citet{doi:10.1073/pnas.2015509117}). We use an example of three classes and denote the last-layer features $\mathbf{h}_{c,i}$, mean class features $\tilde{\mathbf{h}}_{c}$, and last-layer class weight vectors $\mathbf{w}_{c,i}$. Circles denote individual last-layer features, while compound and filled arrows denote class weight and mean feature vectors, respectively. As training progresses, the last-layer features of each class collapse to their corresponding class means (NC1), different class means converge to the vertices of the simplex ETF (NC2), and the class weight vector of the last-layer linear classifier approaches the corresponding class means (NC3).}
\vspace{-0.1in}
\label{fig:ncillu}
\end{figure*}

These observations reveal compelling insights into the symmetry and mathematical preferences of over-parameterized neural network classifiers. Subsequently, further work has demonstrated that \NC may play a significant role in the generalization, transfer learning (\citet{galanti2022role}), depth minimization (\citet{depthminimization}), and implicit bias of neural networks (\citet{poggio2020explicit}). 


Our paper is motivated by the following two questions:
\ifdoublecol
\else
\begin{enumerate}
\setlength\itemsep{0.02em}
    \vspace{0.02in}
    \item What is a minimal set of conditions that would guarantee the emergence of \NC?
    \vspace{0.02in}

    \item Can \NC provide new insight into understanding some widely used training techniques, such as batch normalization and weight decay?
\end{enumerate}
\fi


\subsection{Main Results}


We consider deep neural networks trained using cross-entropy (CE) loss on a balanced dataset. Our asymptotic theoretical analysis shows that last layer \textit{batch normalization, weight decay, and near-optimal cross-entropy loss} constitutes sufficient conditions for several core properties of \NC. Furthermore, the presence of \NC becomes more evident with a larger WD parameter (up to a limit) and smaller loss under the presence of BN, which is substantiated by extensive experiments that demonstrate improving \NC measures with lowering loss, increasing weight decay parameter, and decreasing last-layer feature norm.


To emphasize the geometric intuition of \NC, we use cosine similarity to measure the proximity to the \NC structure. Specifically, NC1 implies that the feature vectors in each class $c$ collapse to the same vector and achieve average feature cosine similarity of features from the same class $\text{intra}_c = 1$. NC2 implies that the class feature means achieves the maximal angle configuration, and thus the inter-class feature cosine similarity for any two classes $c,c'$ satisfies $\text{inter}_{c,c'} = -\frac{1}{C-1}$ (a property of the simplex ETF structure). Our main theorem states that, in the near-optimal regime, the intra-class and inter-class cosine similarity measures of batch-normalized models, which demonstrate the feature vectors' proximity to the \NC structure, can be quantitatively bounded by a function of the weight decay parameter $\lambda$ and loss value $\epsilon$ (with the class number $C$ constant when given target task).

\begin{theorem}[Informal version of Theorem~\ref{thm:main_cor}]
\label{thm:informal_main}
    For the layer-peeled classification model of $C$ classes with weight decay parameter $\lambda$ and cross-entropy training loss within $\epsilon$ of the optimal loss, the following holds for most classes/pairs of classes:
    \begin{enumerate}
        \item \textit{(NC1)} The average intra-class feature cosine similarity of class $c$:
        $$\text{intra}_c\geq 1-O\left((C/\lambda)^{O(C)}\sqrt{\epsilon}\right),$$
        \item \textit{(NC2)} The average inter-class feature cosine similarity of the class pair $c, c'$:
        $$\text{inter}_{c,c'}\leq -\frac{1}{C-1} +O\left((C/\lambda)^{O(C)}\epsilon^{1/6}\right).$$
    \end{enumerate}
\end{theorem}

We complement the theoretical findings with experiments on both synthetic and real datasets to investigate the factors that influence \NC. As expected, we observe that BN, increased WD, and reduced training loss contributes to the occurrence of \NC.

Our main contributions can be summarized as follows:
\begin{itemize}
    
    \item \textbf{\NC Proximity Bound under Near-optimal Loss with Cosine Similarity Measure and Worst Case Analysis.} 
    By adopting the geometrically intuitive cosine similarity measure, we prove quantitative \NC bounds in the \textit{near-optimal regime}, which avoids less realistic assumptions of achieving exact optimal loss. Furthermore, we focus on the worst class \NC measure, uncovering insights that the global average analysis in prior work does not readily reveal.

    \item \textbf{Role of Weight Decay and Batch Normalization.} We offer a novel viewpoint for understanding the roles of WD and BN through the lens of \NC as a catalyst for learning more compact features for the same class. Theoretically, we demonstrate that BN and large WD lead to better guarantees of \NC by regularizing the norms of feature and weight matrices.
    Empirically, our findings further verify that \NC is most significant with BN and high WD values. 


\end{itemize}

\subsection{Related Work}

\par \textbf{Neural Collapse.} Our work closely relates to recent studies that analyze \NC utilizing the layer-peeled model or unconstrained feature model (\citet{unconstrained}). Following this model, several works have demonstrated that solutions satisfying \NC are the only global optimizers when trained using either CE (\citet{ji2022unconstrained, zhu2021geometric, LU2022224}) or Mean Squared Error (MSE) loss 
(\citet{centralpath, zhou2022optimization}). Our work goes beyond the global optimizer by quantitatively analyzing \NC in the near-optimal regime, and consequently studying the factors that affect \NC. 

Another line of work focuses on analyzing the training dynamics and optimization landscape using the unconstrained feature model (UFM) (\citet{unconstrained, zhu2021geometric, ji2022unconstrained, centralpath, normgeometric}). These works establish that, under both CE and MSE loss, the UFM presents a benign global optimization landscape. As a result, following gradient flow or first-order optimization methods tend to yield solutions that fulfill \NC. However, the simplification inherent in the UFM introduces a significant disparity between theory and reality. Specifically, optimizing weights in the earlier layers of a network can lead to outcomes markedly different from those achieved by direct optimization of the last-layer features. In contrast, our findings are \textit{optimization-agnostic} and applicable when direct optimization of the last-layer features is unfeasible.

Due to the space limit, we cannot accommodate all related works in understanding \NC and refer readers to \citep{survey} and appendix Table \ref{appendix_table_comparison} for a more comprehensive survey and comparison with our work.

\par \textbf{Weight Decay.}
The concept of WD or $\ell_2$ regularization originates from early research in the stability of inverse problems (\citet{tikhonov1943stability}), and has since been extensively investigated in the field of statistics (\citet{hoerl1970ridge}). In the context of neural networks, WD serves as a constraint of the network capacity (\citet{goodfellow2016deep}). Several studies have demonstrated that WD enhances the model generalization by suppressing irrelevant weight vector components and diminishing static noise in the targets (\citet{krogh1991simple,shalev2014understanding}). Additionally, various studies regard WD as a mechanism that favorably affects optimization dynamics. Several works contribute to the success of WD in changing the effective learning rate (\citet{van2017l2,li2020understanding,li2020reconciling}). 
\citet{andriushchenko2023we} demonstrates that WD improves the balance in the bias-variance optimization tradeoff, which leads to lower training loss.

\par \textbf{Batch Normalization.}
BN was first introduced by \citet{bn} to address the issue of internal covariate shift in deep neural networks. \citet{liao2016importance} argues that BN mitigates the ill-conditioning problem as the network depth increases. \citet{luo2018towards} decomposes BN intro population normalization and an explicit regularization. Numerous empirical studies have demonstrated BN's positive effects on the optimization landscape
through large-scale experiments (\citet{bjorck2018understanding, santurkar2018does,kohler2019exponential}). \citet{yang2019mean} shows that BN regularizes the gradients and improves the optimization landscape using mean field theory. More recently,
\citet{balestriero2022batch} explores BN from the perspective of function approximation, arguing that BN adapts the geometry of network's spline partition to match the data.

\section{Theoretical Results}
\label{sec:theory}
\subsection{Problem Setup and Notations}

\paragraph{Neural Network with Cross-Entropy (CE) Loss.} In this work, we consider neural network classifiers without bias terms trained using CE loss on a balanced dataset. A vanilla deep neural network classifier is composed of a feature representation function $\boldsymbol{h}^{(L)}(\boldsymbol{x})$ and a linear classifier parameterized by $\mathbf{W}^{(L)}$. Specifically, an $L$-layer vanilla deep neural network can be mathematically formulated as:
\ifdoublecol
\else
$$f(\boldsymbol{x};\boldsymbol{\theta})=\underbrace{\boldsymbol{W}^{(L)}}_{\text {Last layer weight $\mathbf{W}=\mathbf{W}^{(L)}$}} \underbrace{BN\left(\sigma\left(\boldsymbol{W}^{(L-1)} \cdots \sigma\left(\boldsymbol{W}^{(1)} \boldsymbol{x}+\boldsymbol{b}^{(1)}\right)+\dots+\boldsymbol{b}^{(L-1)}\right)\right)}_{\text {last-layer feature } \boldsymbol{h}=\phi_{\boldsymbol{\theta}}(\boldsymbol{x})}.$$
\fi

Each layer is composed of an affine transformation parameterized by weight matrix $\boldsymbol{W}^{(l)}$ followed by a non-linear activation $\sigma$ such as $\text{ReLU}(x)=\max\{x,0\}$ and BN.

The network is trained by minimizing the empirical risk over all samples $\left\{\left(\boldsymbol{x}_{c, i}, \boldsymbol{y}_c\right)\right\}, c\in [C], i\in [N]$ where each class contains $N$ samples and $\boldsymbol{y}_c$ is the one-hot encoded label vector for class $c$. We also denote $\h_{c,i}=\boldsymbol{h}(\boldsymbol{x}_{c,i})$ as the last-layer feature corresponding to $\boldsymbol{x}_{c,i}$. The training process minimizes the average CE loss 
\ifdoublecol
\else
$$\mathcal{L}=\frac{1}{CN}\sum_{c=1}^C \sum_{i=1}^{N} \mathcal{L}_{\mathrm{CE}}\left(f(\boldsymbol{x}_{c,i};\boldsymbol{\theta}), \boldsymbol{y}_c\right)=\frac{1}{CN}\sum_{c=1}^C \sum_{i=1}^{N} \mathcal{L}_{\mathrm{CE}}\left(\boldsymbol{Wh}_{c,i}, \boldsymbol{y}_c\right),$$
\fi

where the cross entropy loss function for a one-hot encoding $\boldsymbol{y}_c$ is:
$$\mathcal{L}_{\mathrm{CE}}(\boldsymbol{z}, \boldsymbol{y}_c)=-\log\left(\frac{\exp(z^{(c)})}{\sum_{c'=1}^{C}\exp(z^{(c')})}\right).$$

\paragraph{Batch Normalization and Weight Decay.} For a given batch of vectors $\{\V_1, \V_2, \cdots, \V_b\}\subset \R^d$, let $v^{(k)}$ denote the $k$'th element of $\V$. BN developed by \citet{bn} performs the following operation along each dimension $k\in[d]$:
$$BN(\V_i)^{(k)}=\frac{v_i^{(k)}-\mu^{(k)}}{\sigma^{(k)}}\times \gamma^{(k)}+b^{(k)}.$$
Where $\mu^{(k)}$ and $(\sigma^{(k)})^2$ are the mean and variance along the $k$'th dimension of all vectors in the batch. The vectors $\boldsymbol{\gamma}$ and $\boldsymbol{b}$ are trainable parameters that represent the desired variance and mean after BN. In our work, we consider BN layers without bias (i.e. $\boldsymbol{b}=0$).

WD is a technique in deep learning training that regularizes neural network weights. Specifically, the Frobenius norm of each weight matrix $\boldsymbol{W}^{(l)}$ and BN weight vector $\boldsymbol{\gamma}^{(l)}$ is added as a penalty term to the final loss. Thus, the regularized loss function with WD parameter $\lambda$ is
\begin{align}
\mathcal{L}_{\mathrm{reg}}=\mathcal{L}+\frac{\lambda}{2}\sum_{l=1}^L(\|\boldsymbol{\gamma}^{(l)}\|^2+\|\mathbf{W}^{(l)}\|_F^2),\label{loss_form}
\end{align}
 We consider the simplified layer-peeled model that only applies WD regularization to the network's final linear and BN layer. Under this setting, the regularized loss is:
\begin{align}
\mathcal{L}_{\mathrm{reg}}=\mathcal{L}+\frac{\lambda}{2}(\|\boldsymbol{\gamma}\|^2+\|\mathbf{W}\|_F^2),\label{loss_pelled}
\end{align}
where $\mathbf{W}$ is the last layer weight matrix and $\boldsymbol{\gamma}$ is the weight of the BN layer before the final linear transformation.

\subsection{Cosine Similarity Measure of Neural Collapse}
\label{cossimm}
Numerous measures of NC have been used in past literature, including within-class covariance (\citet{doi:10.1073/pnas.2015509117}), signal-to-noise (SNR) ratio (\citet{centralpath}), as well as class distance normalized variance (CDNV, \citet{galanti2022role}).  In this work, we focus on the cosine similarity measure (\citet{Kornblith2020WhyDB}) of \NC, which emphasizes simplicity and geometric interpretability at the cost of discarding norm information. Cosine similarity is widely used as a measure between features of different samples in both practical feature learning and machine learning theory.

The average intra-class cosine similarity of class $c$ is defined as:
$$\mathit{intra}_c=\frac{1}{N^2}\sum_{i=1}^N\sum_{j=1}^N \cos_{\angle}(\mathbf{h}_{c,i}-\tilde{\mathbf{h}}_G, \mathbf{h}_{c,j}-\tilde{\mathbf{h}}_G),$$
where
$$\cos_{\angle}(\mathbf{x}, \mathbf{y})=\frac{\mathbf{x}^\top\mathbf{y}}{\|\mathbf{x}\|\cdot \|\mathbf{y}\|}, \quad \tilde{\mathbf{h}}_G=\mathop{\text{Avg}}\limits_{c,i}\{\mathbf{h}_{c,i}\}.$$
Similarity, the inter-class cosine similarity between two classes $c,c'$ is defined as:
$$inter_{c,c'}=\frac{1}{N^2}\sum_{i=1}^N\sum_{j=1}^N \cos_{\angle}(\mathbf{h}_{c,i}-\tilde{\mathbf{h}}_G, \mathbf{h}_{c',j}-\tilde{\mathbf{h}}_G)$$
In our theoretical analysis, we consider batch normalized last layer features without the bias term, and thus the global mean $\tilde{\mathbf{h}}_G$ is guaranteed to be zero and thus can be discarded.
\paragraph{Relationship with \NC.} While cosine similarity does not measure vector norms, it can describe {\it necessary} conditions for the core observations of \NC as follows:
\begin{enumerate}
    \renewcommand{\labelenumi}{(NC\arabic{enumi})}
    \item {\it (Variability Collapse)} All features in the same class collapse to the class mean and must achieve an intra-class cosine similarity $intra_c\rightarrow 1$.
    \item {\it (Convergence to Simplex ETF)} Class means converge to the vertices of a simplex ETF, which implies that $inter_{c,c'}\rightarrow -\frac{1}{C-1}$.
    \item {\it (Convergence to Self-Duality)} Centered class weights $\dot{\mathbf{w}}_c$ and their corresponding features $\tilde{\mathbf{h}}_c$ converge to each other up to rescaling, i.e., $\cos_{\angle}(\dot{\mathbf{w}}_c, \tilde{\mathbf{h}}_c)\rightarrow 1$.
\end{enumerate}   
As \citet{doi:10.1073/pnas.2015509117} has shown that NC4 is a corollary of NC1-3, we will also mainly focus on NC1-3.

\subsection{Main Results}
Before presenting our main theorem (Theorem~\ref{thm:informal_main}) on BN and WD, we first present a more general preliminary theorem that provides theoretical bounds for the intra-class and inter-class cosine similarity for any classifier with near-optimal (unregularized) CE loss. 
Our first theorem states that if the average last-layer feature norm and the last-layer weight matrix norm are both {\it bounded}, then achieving {\it near-optimal loss} implies that {\it most classes} have intra-class cosine similarity near one and {\it most pairs of classes} have inter-class cosine similarity near $-\frac{1}{C-1}$. 
\begin{theorem}[\NC proximity guarantee with bounded norms]
    \label{main}
    For any neural network classifier without bias trained on a dataset with the number of classes $C\geq 3$, samples per class $N\geq 1$, and the last layer feature dimension $d\geq C$. Under the following assumptions:
    \begin{enumerate}
        \item The quadratic average of the last-layer feature norms $\sqrt{\frac{1}{CN} \sum_{c=1}^C\sum_{i=1}^N \|\mathbf{h}_{c,i}\|^2}\leq \alpha$.
        \item The Frobenius norm of the last-layer weight $\|\mathbf{W}\|_F\leq \sqrt{C}\beta$.
        \item The average cross-entropy loss over all samples $\mathcal{L}\leq m+\epsilon$ for small $\epsilon>0$.
    \end{enumerate}
    Here $m=\log(1+(C-1)\exp(-\frac{C}{C-1}\alpha\beta))$ is the minimum achievable loss under the norm constraints.
    Then for at least $1-\delta$ fraction of all classes, with $\frac{\epsilon}{\delta}\ll1$, there is
    $$\text{intra}_c\geq 1-O\left(\frac{e^{O(C\alpha\beta)}}{\alpha\beta}\sqrt{\frac{\epsilon}{\delta}}\right),$$
    $$\cos_{\angle}(\dot{\mathbf{w}}_c, \tilde{\mathbf{h}}_c)\geq 1-O(e^{O(C\alpha\beta)}\sqrt{\frac{\epsilon}{\delta}}),$$
    and for at least $1-\delta$ fraction of all pairs of classes $c,c'$, with $\frac{\epsilon}{\delta}\ll1$, there is
    $$\text{inter}_{c,c'}\leq -\frac{1}{C-1} +O\left(\frac{e^{O(C\alpha\beta)}}{\alpha\beta}(\frac{\epsilon}{\delta})^{1/6}\right).$$
\end{theorem}

    

The quantitative bounds of our theorem imply that smaller last-layer feature and weight norms can provide stronger guarantees on \NC.

The proof of Theorem~\ref{main} is inspired by the optimal-case proof from \citet{LU2022224}, which shows the global optimality conditions using Jensen's inequality. Our proof extends to the near-optimal case by carefully relaxing the three strict Jensen conditions into near-optimal quantitative guarantees and analyzing the dynamics between the resulting Jensen gaps. Specifically, we show in Lemma~\ref{strongjensensub} (based on strongly convex function result from \cite{Merentes2010}) that if a set of variables achieves roughly equal value on the LHS and RHS of Jensen's inequality for a strongly convex function, then the mean of every subset cannot deviate too far from the global mean.

\begin{lemma}[Subset mean close to global mean by Jensen's inequality on strongly convex functions]
    \label{strongjensensub}
    Let $\{x_i\}_{i=1}^N\subset \mathcal{I}$ be a set of $N$ real numbers, let $\tilde{x}=\frac{1}{N}\sum_{i=1}^N x_i$ be the mean over all $x_i$ and $f$ be a function that is $m$-strongly-convex on $\mathcal{I}$. If
    $$\frac{1}{N}\sum_{i=1}^N f(x_i)\leq f(\tilde{x})+\epsilon,$$
    i.e., Jensen's inequality is satisfied with gap $\epsilon$, then for any subset of samples $S\subseteq [N]$, let $\delta=\frac{|S|}{N}$, there is $$\tilde{x}+\sqrt{\frac{2\epsilon(1-\delta)}{m\delta}}\geq \frac{1}{|S|}\sum_{i\in S} x_i\geq \tilde{x}-\sqrt{\frac{2\epsilon(1-\delta)}{m\delta}}.$$
\end{lemma}

 This lemma can be a general tool to convert optimal-case conditions derived using Jensen's inequality into high-probability proximity bounds under near-optimal conditions.




We now proceed to the formal version of the main theorem that theoretically demonstrates the relationship between \NC, BN, and WD.

 \begin{theorem}[Formal Version of Theorem~\ref{thm:informal_main}]
    \label{thm:main_cor}
    For a neural network classifier without bias  trained on a dataset with the number of classes $C\geq 3$ and samples per class $N\geq 1$, we consider its layer-peeled model with batch normalization before the final layer with parameter $\boldsymbol{\gamma}$, weight decay parameter $\lambda < 1/\sqrt{C}$ and regularized CE loss
    \[
    \mathcal{L}_{\mathrm{reg}}=\frac{1}{CN}\sum_{c=1}^C \sum_{i=1}^{N} \mathcal{L}_{\mathrm{CE}}\left(\boldsymbol{Wh}_{c,i}, \boldsymbol{y}_c\right)+\frac{\lambda}{2}(\|\boldsymbol{\gamma}\|^2+\|\mathbf{W}\|_F^2)
    \]
    satisfying $\mathcal{L}_{\mathrm{reg}}\leq m_{\mathrm{reg}}+\epsilon$ for small $\epsilon$,
    where $m_{reg}$ is the minimum achievable regularized loss.
    Then for at least $1-\delta$ fraction of all classes, with $\frac{\epsilon}{\delta}\ll1$, $\epsilon < \lambda$ and for small constant $\kappa>0$ and $\rho=(Ce/\lambda)^{\kappa C}$, the intra-class cosine similarity for class $c$
    $$\quad \mathit{intra}_c\geq 1-\frac{C-1}{C}\sqrt{\frac{128\rho\epsilon(1-\delta)}{\delta}}.$$
    The cosine similarity between feature and weight for class $c$
    $$\cos_{\angle}(\dot{\mathbf{w}}_c, \tilde{\mathbf{h}_c})\geq 1-2\sqrt{\frac{2\rho\epsilon(1-\delta)}{\delta}}.$$
    For at least $1-\delta$ fraction of all pairs of classes $c,c'$, with $\frac{\epsilon}{\delta}\ll1$, the inter-class cosine similarity $inter_{c,c'}$
    \vspace{-0.1in}
    \begin{align*}
    \leq -\frac{1}{C-1}+\frac{C\rho}{C-1}\sqrt{\frac{2\epsilon}{\delta}}+4(\rho\sqrt{\frac{2\epsilon}{\delta}})^{1/3}+\sqrt{\rho\sqrt{\frac{2\epsilon}{\delta}}}.
    \end{align*}

\end{theorem}

Since $\rho=(Ce/\lambda)^{\kappa C}$ is a decreasing function of $\lambda$, higher values of $\lambda$ would result in larger lower bounds of $intra_c$ and smaller upper bounds of $inter_{c,c'}$ under the same loss gap $\epsilon$. According such, under the presence of BN and WD of the final layer, larger values of WD provide stronger \NC guarantees in the sense that the intra-class cosine similarity of most classes is closer to 1 and the inter-class cosine similarity of most pairs of classes is closer to $-\frac{1}{C-1}$.




\subsection{Conclusion}

Our theoretical result shows that last-layer BN, last-layer WD, and near-optimal average CE loss are sufficient conditions to guarantee proximity to the \NC structure as measured using cosine similarity, regardless of the training method and earlier layer structure. Moreover, our quantitative bound implies that a larger WD value and smaller loss result in stronger bounds on \NC.

\section{Empirical Results}
\label{expt}
\begin{figure*}
\includegraphics[width=0.9\textwidth]{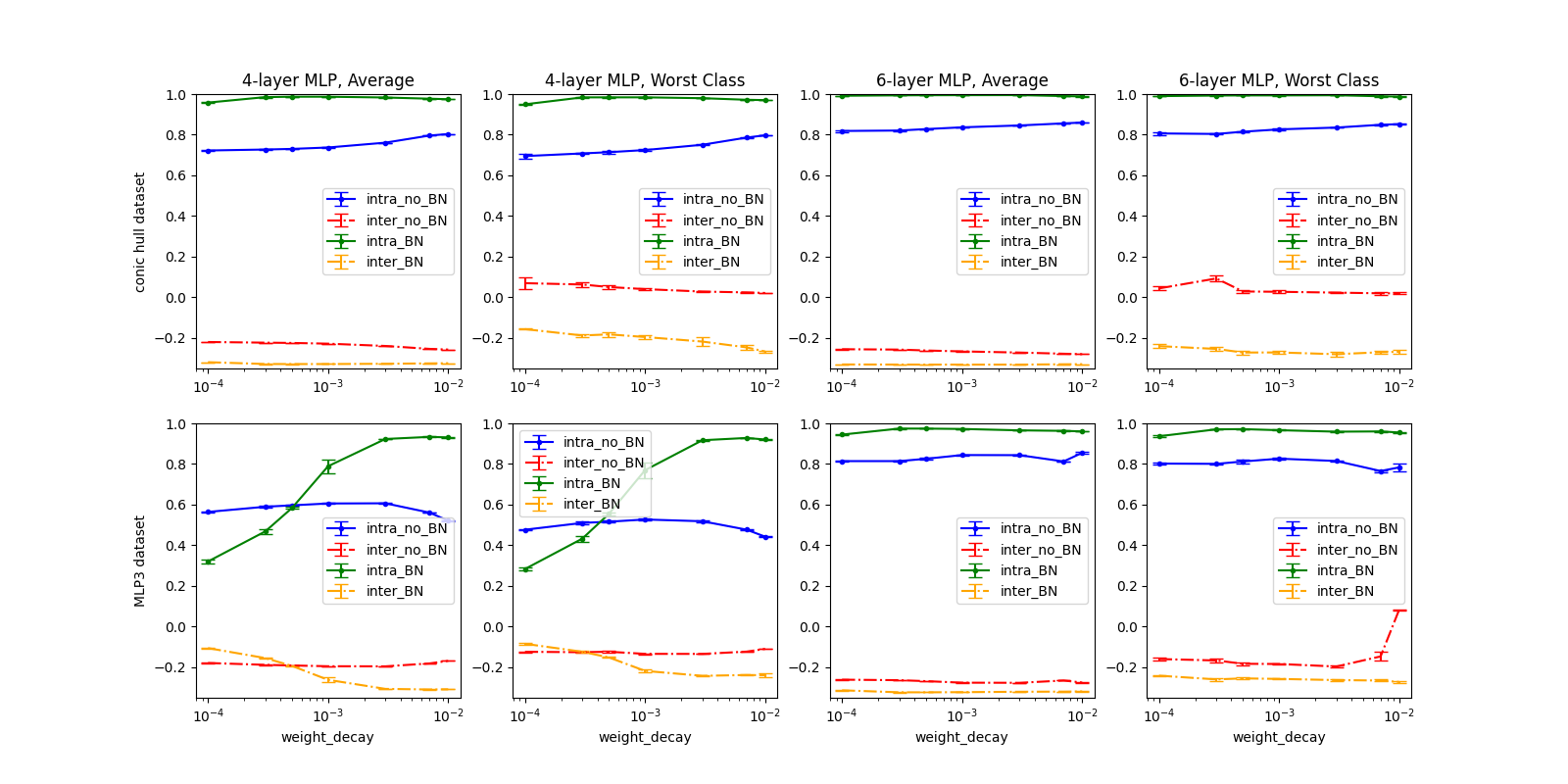}
\vspace{-0.2in}
\caption{\NC increases with WD under BN: Minimum intra-class and maximum inter-class Cosine Similarity for 4-layer and 6-layer MLP under Different WD and BN on the synthetic dataset generated using a randomly initialized 3-layer MLP. Higher values of intra-class and lower values of inter-class cosine similarity imply a higher degree of Neural Collapse. The {\bf \textcolor{green}{green}} and {\bf \textcolor{yellow}{yellow}} lines are cosine similarity measures for the model with BN, which demonstrates stronger \NC along with higher WD values. Standard deviation over 5 experiments.}
\label{fig:neural_collapse}

\end{figure*}


In this section, we present extensive empirical evidence to complement our theoretical discoveries. Specifically, our experiments highlight the significance of BN and WD in the emergence of \NC by suggesting that: 
\ifdoublecol
\else
\begin{itemize}
    \item The degree of \NC is most significant under the presence of BN and high WD values.
    \item The degree of \NC improves with decreasing loss during training more steadily under the presence of BN.
    \item The degree of \NC is more significant at lower last-layer feature norm values.
\end{itemize}
\fi

\subsection{Setup}
\label{setup}
We perform experiments on both synthetic and real-world datasets.
\par{\textbf{Synthetic Datasets.}}
\label{setup_syn}
Our first set of experiments uses a vanilla neural network (i.e., Multi-Layer Perceptron with ReLU activation) to classify well-defined synthetic datasets of different distribution complexities. We aim to use straightforward model architectures and well-defined distributions to explore the effect of different hyperparameters in \NC under a controlled setting. We consider MLP models with and without BN. In BN models, one BN layer is located after the last ReLU activation and before the final linear transformation.

\par Our first dataset is the conic hull dataset, where the feature space $\mathbb{R}^d$ is separated into $C$ classes using $\lceil \log C \rceil$ randomly generated hyperplanes. Since every pair of classes is linearly separable, neural networks can find a set of weights that perfectly classify all data. Thus, the conic hull dataset is a great starting point for understanding deep classification models. In our experiments, we use class number $C=4$, dimension $d=16$, and training dataset size $N=8000$.

\par We also perform experiments on a more complex dataset where the class labels are generated using a randomly initialized MLP. We ensure that the number of layers and parameters within this data-generator MLP is less than any model used for training. The number of classes, dimensions, and training samples we use are identical to the conic hull dataset.

\paragraph{Real-World Datasets. (Results in Appendix Section \ref{sec:appendexpt})}
Our next set of experiments explores the effect of BN and WD using standard computer vision datasets MNIST (\citet{mnist}), CIFAR-10, CIFAR-100 (\citet{cifar10}), and ImageNet32 (\citet{imagenet}). We use VGG11 and VGG19 (\citet{vgg}) convolutional neural networks as the architecture. Similar to the synthetic experiments, we consider the models with and without BN. The BN model incorporates a BN layer after selected convolution layers. Both models are official implementations of the PyTorch Library.

\paragraph{\textbf{Measures of proximity to the \NC structure.}}
Our experiments adopt the geometrically intuitive cosine similarity measure of \NC as in our theoretical results. While most prior empirical works of \NC focus on the average measures of NC over all classes, (e.g., \citet{doi:10.1073/pnas.2015509117,ji2022unconstrained}), we additionally measure the stricter {\it minimum} intra-class and {\it maximum} inter-class (i.e. the {\bf worst-case} measure over all classes/pairs of classes). When the number of classes is large, the difference between the average and worst-case measures can be very significant and reveal further insights into the details of the feature geometric configuration, as later demonstrated in our experiments.


\begin{figure*}
    \includegraphics[width=0.95\textwidth]{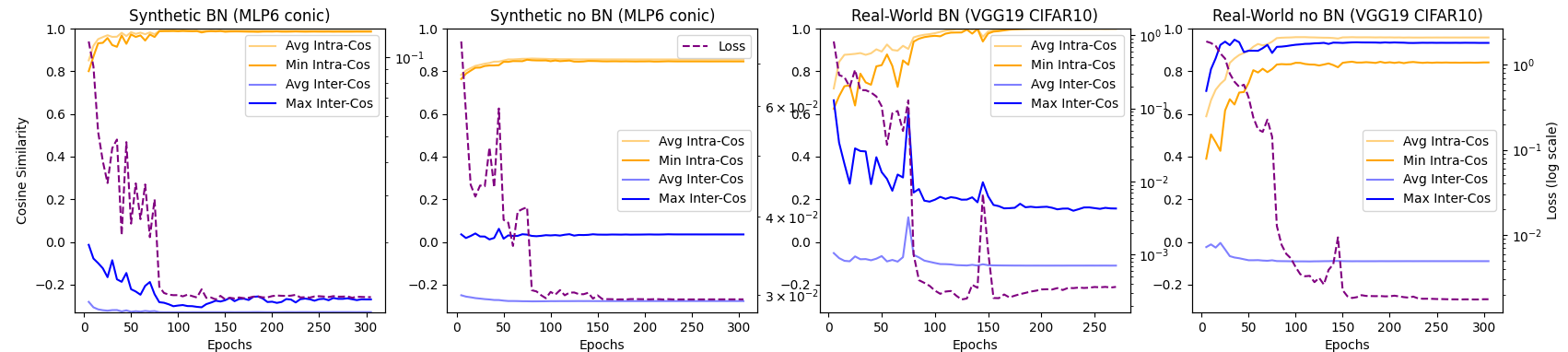}
\vspace{-0.1in}
\caption{\NC closely represents loss value under BN: Relationship between \NC and training loss during the training process. The purple dashed line is the training loss presented in the log scale with axis labels on the right. The models with Batch Normalization (plots 1 and 3) demonstrate more correlation between loss value and \NC during training.}
\label{figure_exp2}
\end{figure*}

\subsection{Relationship with the Presence of BN and WD}
\label{section:bnwd}
In our first set of experiments, we explore the degree of \NC under different presences of BN and values of WD. We conduct experiments on both synthetic and real-world data as described in section \ref{setup} with WD values varying between $10^{-4}$ and $10^{-2}$. Our experimental results for synthetic datasets are presented in Figure \ref{fig:neural_collapse}, while those for real-world datasets can be found in appendix section ~\ref{sec:appendixbnwd}.

Our experiments show that, in both synthetic and realistic scenarios, the highest level of \NC is achieved by models with BN and appropriate WD. Moreover, BN allows the degree of \NC to increase smoothly along with the increase of WD within the range of perfect interpolation, while the degree of \NC is unstable or decreases with the increase of WD in non-BN models. Such a phenomenon is also more pronounced in simpler neural networks and easier classification tasks than in realistic classification tasks.

\begin{figure*}[h]
\centering
\includegraphics[width=1.03\textwidth]{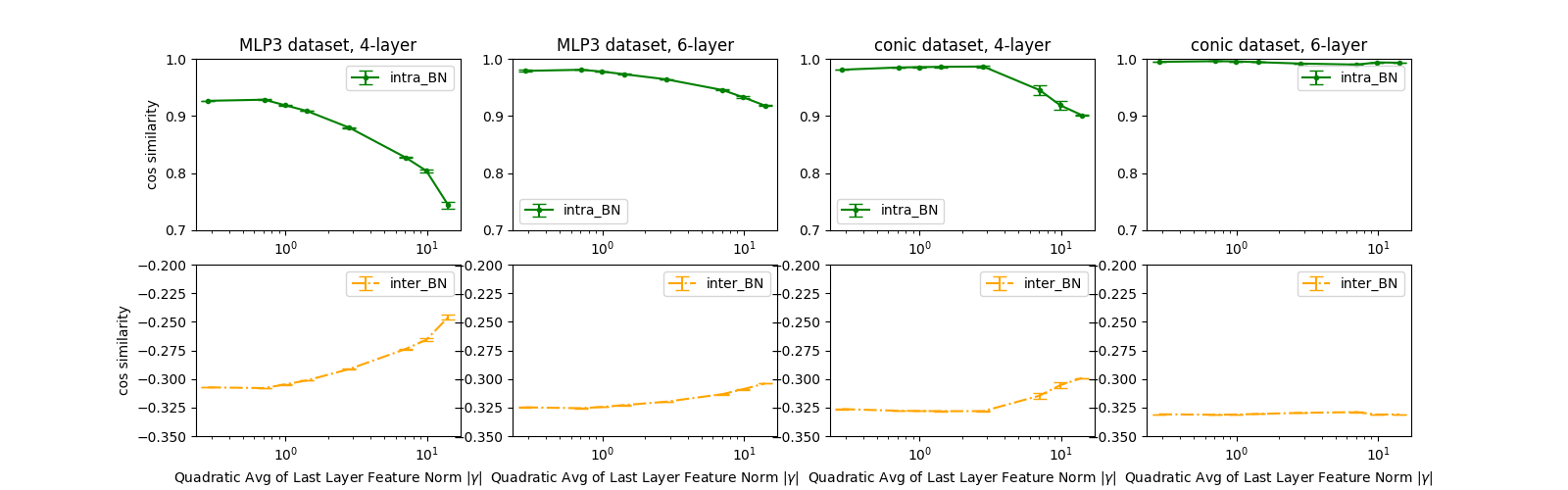}
\vspace{-0.1in}
\caption{\NC correlates with feature norm: Min intra-class and max inter-class Cosine Similarity for synthetic dataset and MLP models with BN under different $|\boldsymbol{\gamma}|$ values. Higher intra-class and lower inter-class cosine similarity indicate a higher degree of \NC. Note that the intra-class and inter-class cosine similarity are split into two plots to display more detailed changes. Except for the 6-layer MLP trained on the conic hull dataset, all settings demonstrate a negative correlation between proximity to \NC and the last-layer feature norm value as constrained by $|\boldsymbol{\gamma}|$. Standard Deviation over 3 experiments.}
\label{fig:plot_fix_const}
\end{figure*}

\subsection{Relationship with Training Loss}
\label{section:loss}
Our next set of experiments explores the emergence of \NC as the training loss decreases during the training process. Specifically, we focus on the evolution of minimum intra-class and maximum inter-class cosine similarity during training. Theorem~\ref{thm:main_cor} implies that, under the presence of BN and WD, the bound on \NC scales with the loss optimality gap $\epsilon$. However, it does not provide guarantees without the presence of BN layers. As such, we hypothesize that the presence of BN layers facilitates the formation of the $\NC$ structure during training as the training loss decreases. Specifically, we record the models' cosine similarity measure every five epochs during training for both models with and without BN.

\par We present our results in Figure~\ref{figure_exp2}. We note that for the synthetic dataset experiment with BN, the degree of \NC demonstrates a strong correlation with training loss (purple dashed line) throughout the training process while the model without BN observes little change in the \NC beyond the first few epochs even though the loss keeps decreasing later on into the training process. For real-world experiments, the model with BN continues to demonstrate a significant correlation between training loss and \NC, while the model without BN observes an increase (instead of the expected decrease) in maximum inter-class cosine similarity during the first phases of training despite a decrease in training loss. Additional experiments with synthetic data under different WD values and real-world data are in Appendix Section \ref{section:appendloss}.
the supplemental materials.

\subsection{Relationship with Feature Norm}
\label{section:featnorm}
Note that Theorem \ref{main} implies that higher feature norm (i.e. $\alpha$) yields stronger theoretical bounds on the degree of \NC. Inspired by this result, we directly investigate the relationship between the proximity of \NC and the last-layer feature norm. Specifically, we set the weight vector of the BN layer (i.e. $\boldsymbol{\gamma}$ in (\ref{loss_pelled})) to a constant value fixed during training. We then compare the cosine similarity measure of \NC under different $|\boldsymbol{\gamma}|$ values. We hypothesize that lower $|\boldsymbol{\gamma}|$ values would induce stronger neural collapse at the terminal phase of training, assuming a small training loss is achieved, and a higher WD value facilitates \NC by inducing smaller $|\boldsymbol{\gamma}|$ value during training. A WD factor of $0.005$ is used for all experiments in this section.

\par We perform this experiment only on synthetic data due to the existence of multiple BN layers in real-world models such as VGG, which makes such operations ambiguous. We vary the constant value set for each entry of the feature vector from $0.02$ to $1$, and the actual $|\boldsymbol{\gamma}|$ value is scaled by a factor of $\sqrt{d}$. Our results are presented in Figure \ref{fig:plot_fix_const}. We note that for most configurations, the cosine similarity of \NC demonstrates a negative correlation with the value of $|\gamma|$. The only exception is the combination of the 6-layer MLP model trained on the conic hull dataset, where the model fits the data so well that near perfect \NC is achieved regardless of the $|\boldsymbol{\gamma}|$ value. Additional experiments with different WD values are in Appendix Section \ref{section:appendfeatnorm}.

\section{Limitations and Future Work}
Our theoretical exploration into deep neural network phenomena, specifically \NC, has its limitations and offers various avenues for further work. Based on our work, we have identified several directions for future efforts:

\begin{itemize}
    \item Our work, like previous studies employing the layer-peeled model, primarily focuses on the last-layer features and posits that BN and WD are only applied to the penultimate layer. \NC has been empirically observed in deeper network layers (\citet{benshaul2022nearest, depthminimization}) and shown to be optimal for regularized MSE loss in deeper unconstrained features models (\citet{tirer2022extended, súkeník2023deep}). An insightful future direction would involve investigating how the proximity bounds to \NC can be generalized to deeper layers of neural networks and understanding how these theoretical guarantees evolve with network depth.
    \item The theoretical model we have developed is idealized, omitting several intricate details inherent to practical neural networks. These include bias in linear layers and BN layers and the sequence of BN and activation layers. 
\end{itemize}

\bibliographystyle{plainnat}
\bibliography{references}
\newpage
\renewcommand{\thesection}{\Alph{section}}
\setcounter{section}{0}

\section{Comparison with other Theoretical Works on the Emergence of \NC}
\label{appendix_table_comparison}
\begin{table}[h]
    \centering
    \label{tab:contribution-comparison}
    \begin{tabular}{l|cccc|ccc}
        \toprule
         & MSE & CE & Reg. & Norm. & Opt. & Landscape & Near-Opt. \\
        \midrule
        \cite{ji2022unconstrained} & & \checkmark & & & \checkmark$^*$ & \checkmark$^*$ & \\
        \cite{zhu2021geometric} & & \checkmark & \checkmark & & \checkmark & \checkmark & \\
        \cite{LU2022224} & & \checkmark &  & \checkmark & \checkmark & & \\
        \cite{poggio2020explicit} & \checkmark & & & \checkmark & \checkmark & \checkmark & \\
        \cite{tirer2022extended} & \checkmark & & \checkmark & & \checkmark & &\\
        \cite{súkeník2023deep} & \checkmark & & \checkmark & & \checkmark & &\\
        \cite{centralpath} & \checkmark & & \checkmark & & \checkmark & \checkmark & \\
        \cite{normgeometric} & & \checkmark &  & \checkmark & \checkmark & \checkmark \\
        \cite{pmlr-v145-e22b} & & \checkmark & & \checkmark & \checkmark & & \\
        This Work & & \checkmark & \checkmark & \checkmark & \checkmark & & \checkmark \\
        \bottomrule
    \end{tabular}
    
    \caption{Comparison with existing theoretical works on the emergence of \NC. "Reg." denotes weight or feature norm regularization assumption, "Norm." denotes weight or feature norm constraint/normalization, "Opt." denotes optimality conditions, and "Landscape" denotes landscape or gradient flow analysis. $^*$ Shows the direction of gradient flow as it tends towards infinity without normalization/regularization.}
\end{table}

\section{Additional Experiments}
\label{sec:appendexpt}
\subsection{Experiment Details}
Unless otherwise specified, all models are trained on RTX4090 GPUs with learning rate $lr=0.001$ for CIFAR10/100 and $lr=0.0001$ for ImageNet32, which decays by a factor of $0.1$ every $1/4$ of the training epochs. Experiments are trained with the Adam optimizer for 300 epochs with Cross Entropy loss. For CIFAR100 and CIFAR10 experiments, models are trained using 8000 training samples. For ImageNet32, the training sample size is 100k.
\subsection{Relationship of \NC with BN and WD on real-world dataset}
\label{sec:appendixbnwd}
\paragraph{Results for CIFAR10 and CIFAR 100}  In figure \ref{fig:plot_CV} we present experimental results for standard computer vision datasets CIFAR10 and CIFAR100 (\cite{cifar10}) using VGG (\cite{vgg}) networks. We trained on weight decay values of $\lambda=3e-4, 5e-4, 1e-3, 5e-3, 7e-3, 1e-2$ using two VGG implementations with and without BN in the PyTorch (\cite{pytorch}) library. Similar to the synthetic experiments, we consider both the average cosine similarity measures and that of the worst-performing class/pair of classes in terms of $intra_c$ and $inter_{c,c'}$ value. The {\bf \textcolor{green}{green}} and {\bf \textcolor{red}{red}} lines are the intra-class and inter-class cosine similarity measures for the model with BN, respectively.

We observe that, in alignment with our hypothesis, models with BN demonstrate stronger \NC than models without BN (i.e. for intra-class, the {\bf \textcolor{green}{green}} lines with BN are higher than the {\bf \textcolor{blue}{blue}} lines without BN, while the {\bf \textcolor{red}{red}} lines for inter-class cosine similarity $inter_{c,c'}$ are above the {\bf \textcolor{yellow}{yellow}} lines without BN). Furthermore, \NC is more evident as the WD value $\lambda$ increases in BN models, observable as the intra-class cosine similarity ({\bf \textcolor{blue}{blue}}) increases while the inter-class cosine similarity ({\bf \textcolor{red}{red}}) decreases with the increase of WD value.
\begin{figure*}[h]
\centering
\includegraphics[width=1.0\textwidth]{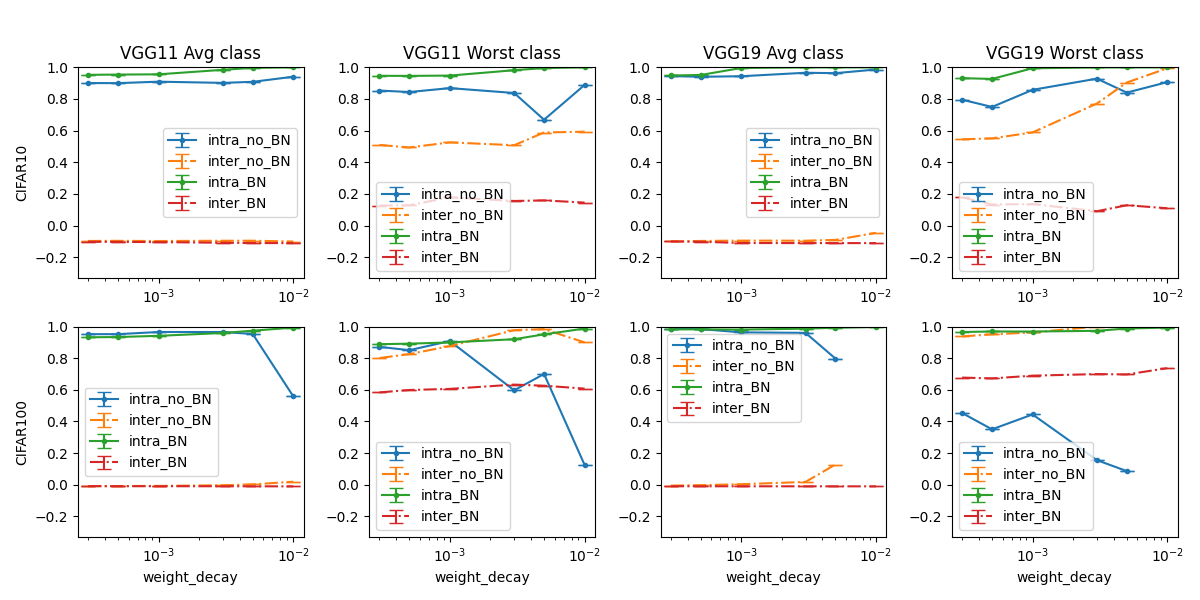}
\caption{Intra-class and Inter-class Cosine Similarity for VGG11 and VGG19 and datasets CIFAR10 and CIFAR 100 under Different WD and BN combinations. Higher intra-class and lower inter-class cosine similarity indicate a higher degree of \NC. Both the average measures over all classes and the worst class are presented. The {\bf \textcolor{green}{green}} and {\bf \textcolor{red}{red}} lines are cosine similarity measures for the model with BN. In most cases, the models with BN demonstrates observably better \NC than non-BN models, and the \NC is more evident in models trained with larger WD value.}
\label{fig:plot_CV}
\end{figure*}
\paragraph{Results for ImageNet32 (1000 classes).} In Figure \ref{fig:plot_ImageNet}, we perform experiments on the ImageNet32 dataset dataset with the VGG11, VGG19 and ResNet Model with BN. The better-performing ResNet model demonstrate the most evident \NC, which increases with the WD parameter. On the other hand, while the VGG models continue to demonstrate increases intra-class cosine similarity with increasing WD, the inter-class cosine similarly also increase, in contrary with our theoretical prediction. This shows that optimization factors takes more precedence than the our optimization-agnostic theoretical bound as the number of classes $C$ increases.
\begin{figure*}[h]
\centering
\includegraphics[width=1.0\textwidth]{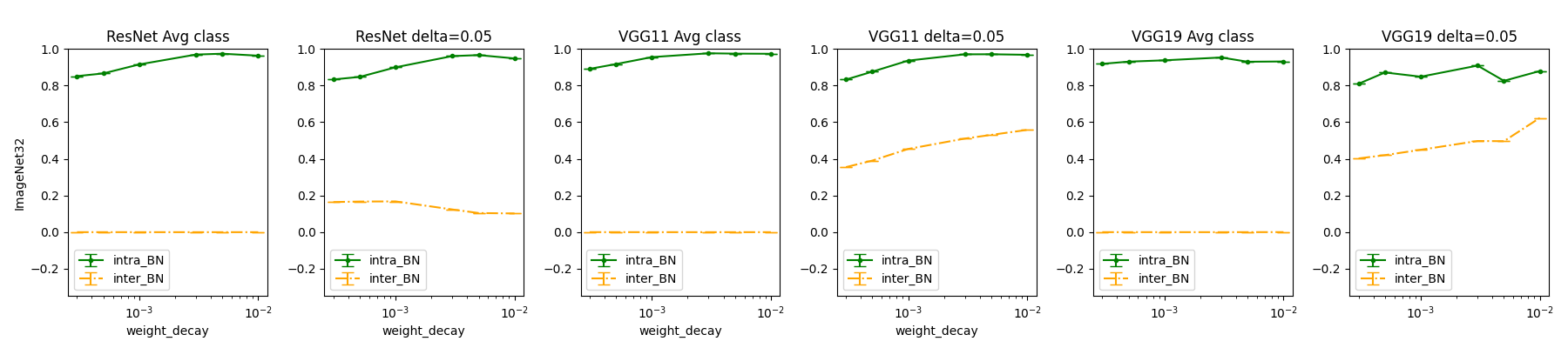}
\caption{Intra-class and Inter-class Cosine Similarity for ImageNet32 under Different WD and BN with different models.. Higher intra-class and lower inter-class cosine similarity indicate a higher degree of \NC. Both the average measures over all classes and the worst class are presented. The {\bf \textcolor{green}{green}} and {\bf \textcolor{yellow}{yellow}} lines are cosine similarity measures for the model with BN.}
\label{fig:plot_ImageNet}
\end{figure*}
\subsection{Relation of \NC with training loss}
\label{section:appendloss}
In main content Section
"Relationship with Training Loss"
we provided one example \NC vs training loss of both synthetic and real-world data. In Figure \ref{fig:losssyn} we provide additional experiments for synthetic data and in Figure \ref{fig:lossreal} we present additional experiments for real-world data and models. Note that most experiments strengthen our claim that BN allows \NC to increase reliably with the minimization of training loss.
\begin{figure}
    \centering
    \includegraphics[width=\textwidth]{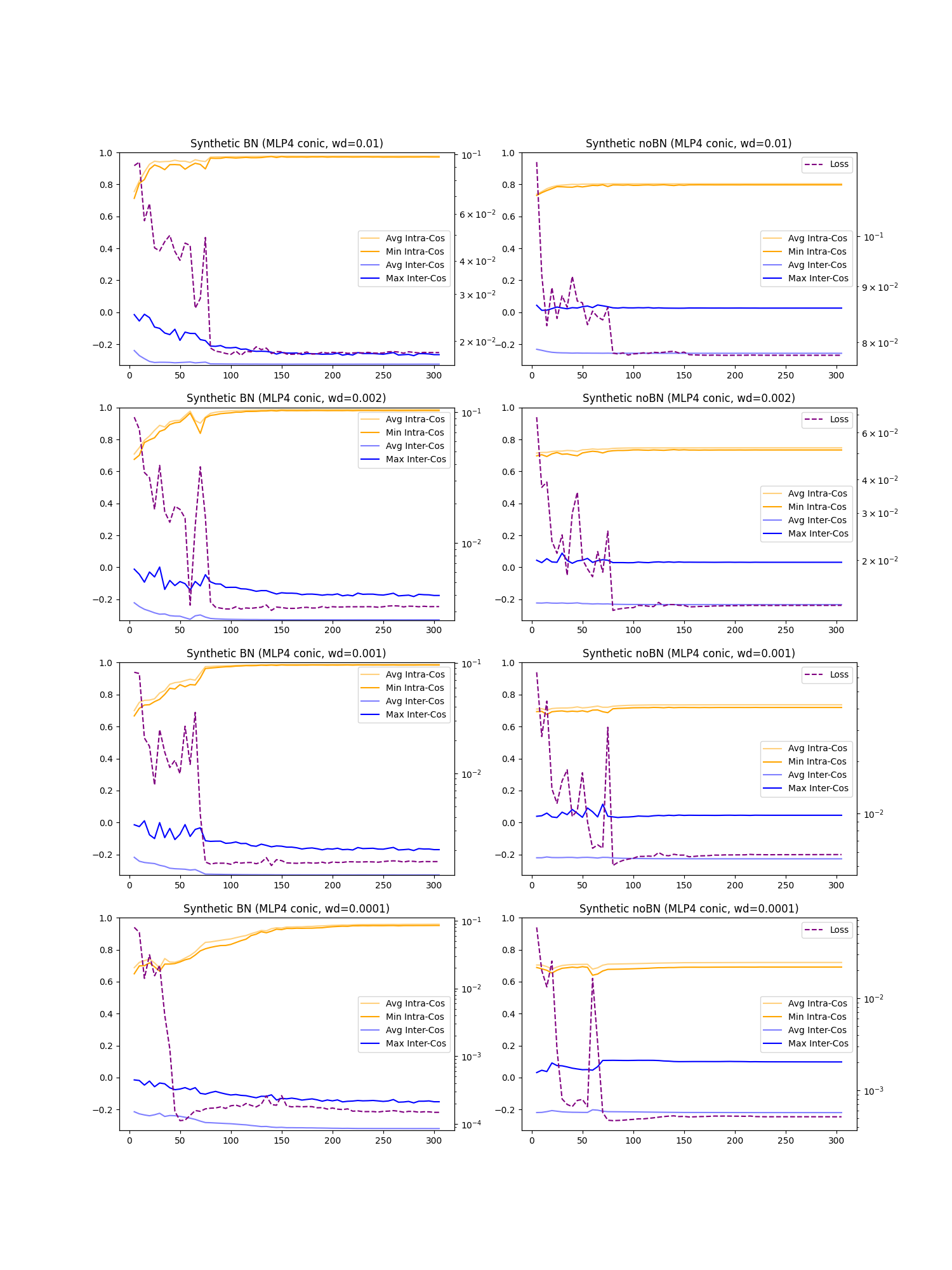}
    \caption{Minimum intra-class cosine similarity and maximum inter-class cosine similarity vs loss during training with different weight decay values using 4-layer MLP trained on the conic hull dataset. Note that the \NC measures barely change during training without BN but increases reliably with loss decrease with BN.}
    \label{fig:losssyn}
\end{figure}
\begin{figure}
    \centering
    \includegraphics[width=\textwidth]{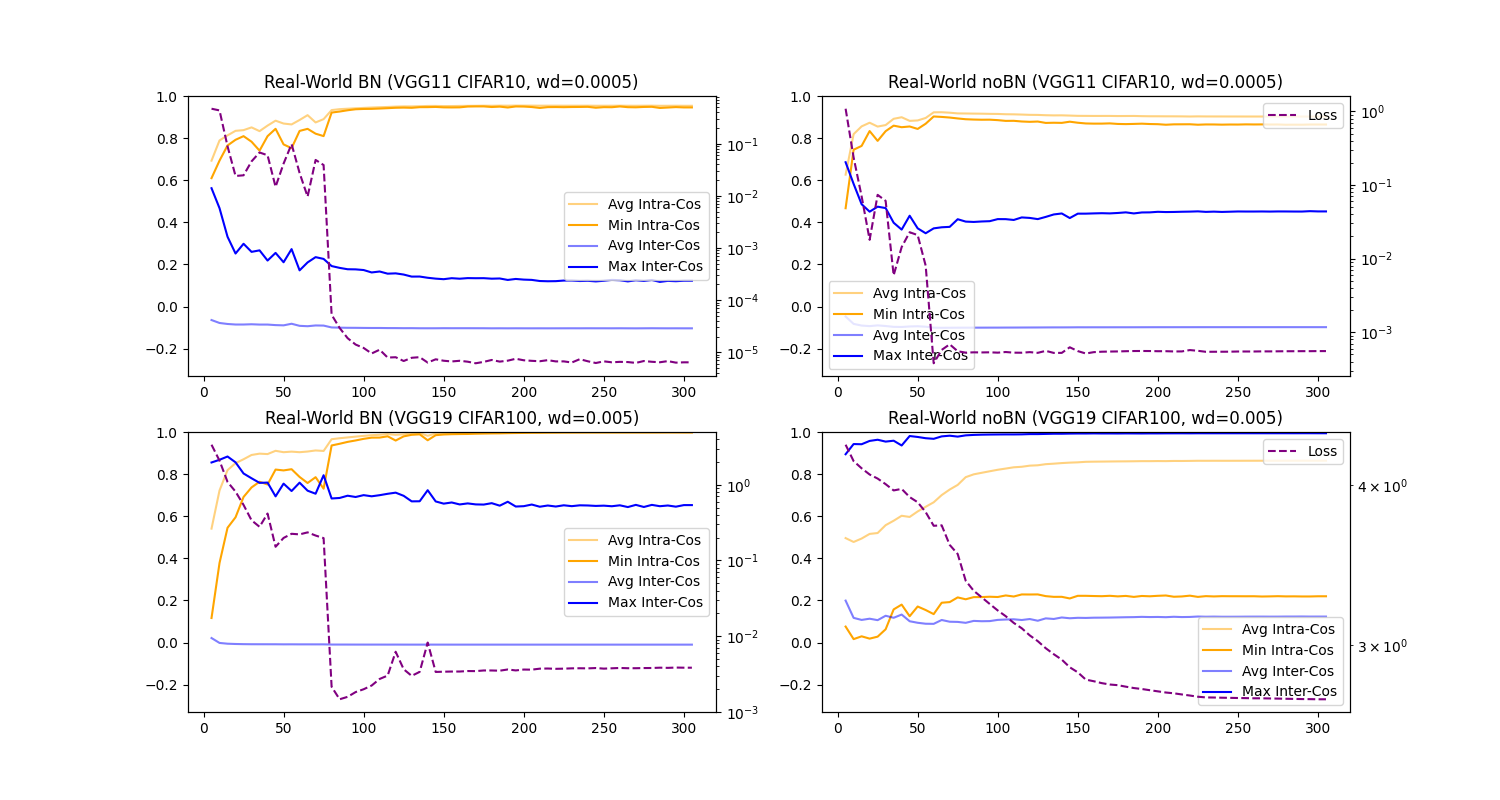}
    \caption{Minimum intra-class cosine similarity and maximum inter-class cosine similarity vs loss during training with real-world data. Note that the \NC measures barely change during training without BN but increases reliably with loss decrease with BN.}
    \label{fig:lossreal}
\end{figure}
\subsection{Relation of \NC with Last-layer Feature Norm}
\label{section:appendfeatnorm}
In main content Section 
"Relation of with Last-layer Feature Norm"
, we presented the result for the relationship of \NC with layer-layer feature norm as parameterized by the norm of the batch norm $\boldsymbol{\gamma}$ vector. We only presented results for weight decay parameter $wd=0.005$. In Figure \ref{fig:featnormwd} we provide additional results for the experiment at a wider range of weight decay values. As indicated by Section \ref{section:bnwd}, lower weight decay parameter results in higher \NC.
\begin{figure}
    \centering
    \includegraphics[width=\textwidth]{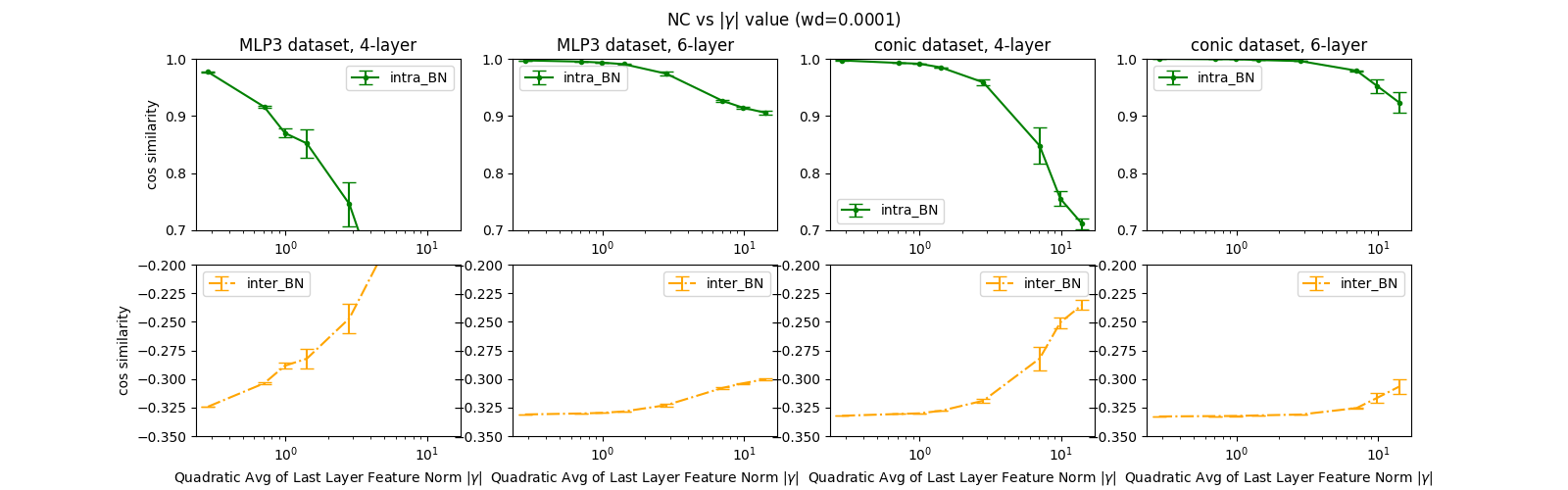}
    \includegraphics[width=\textwidth]{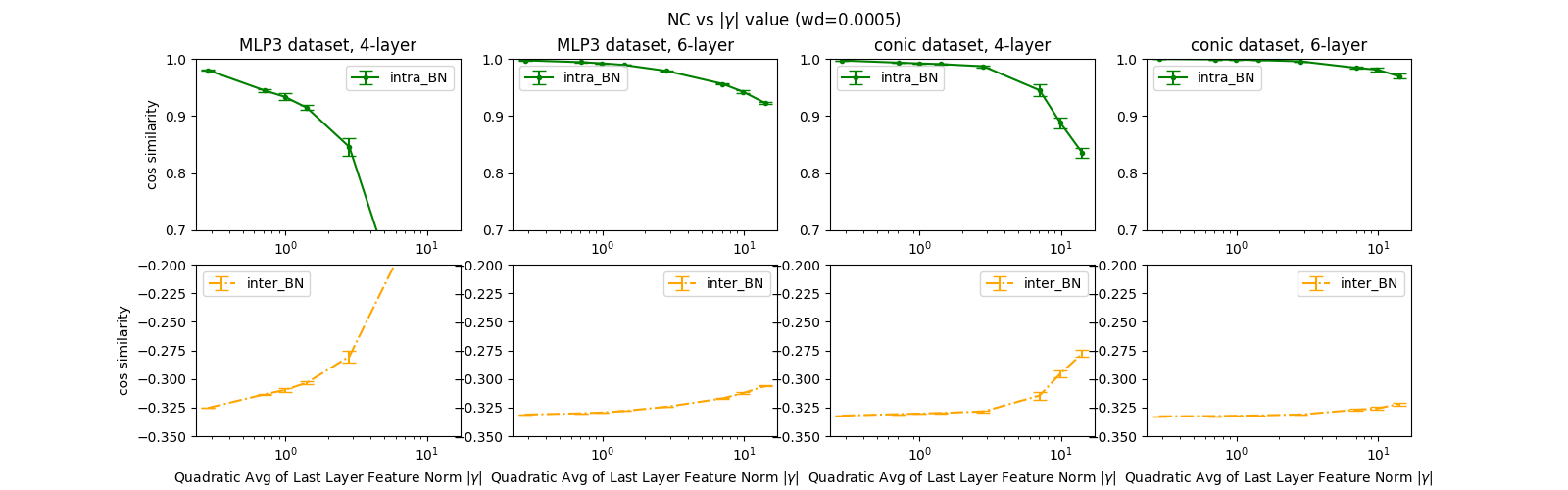}
    \includegraphics[width=\textwidth]{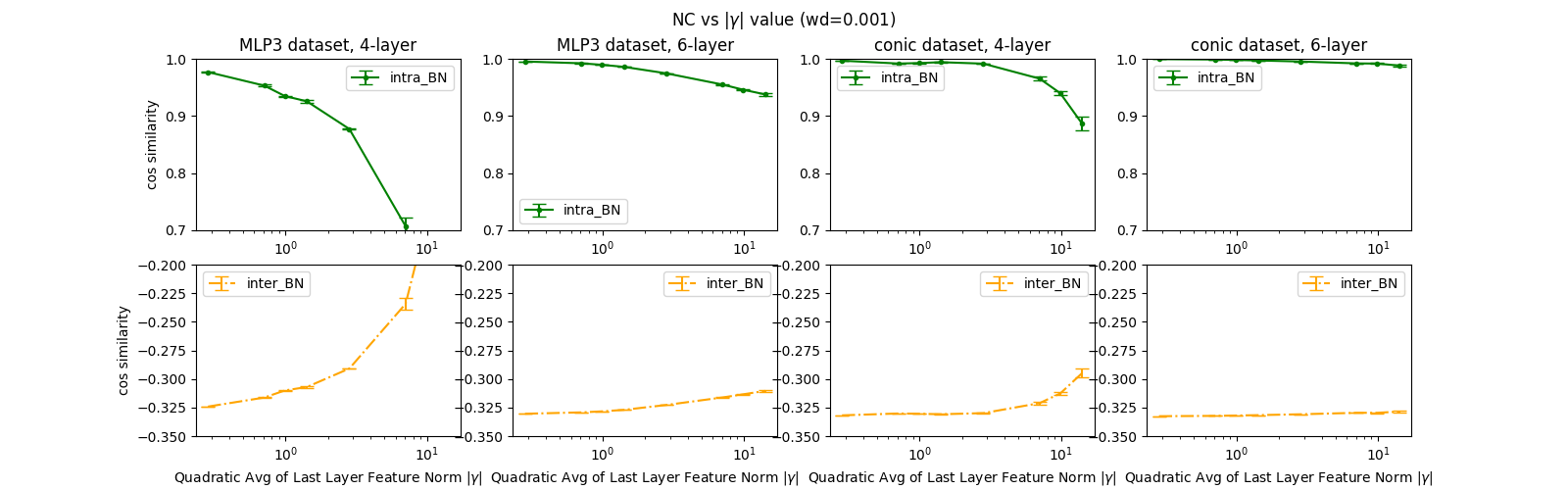}
    \includegraphics[width=\textwidth]{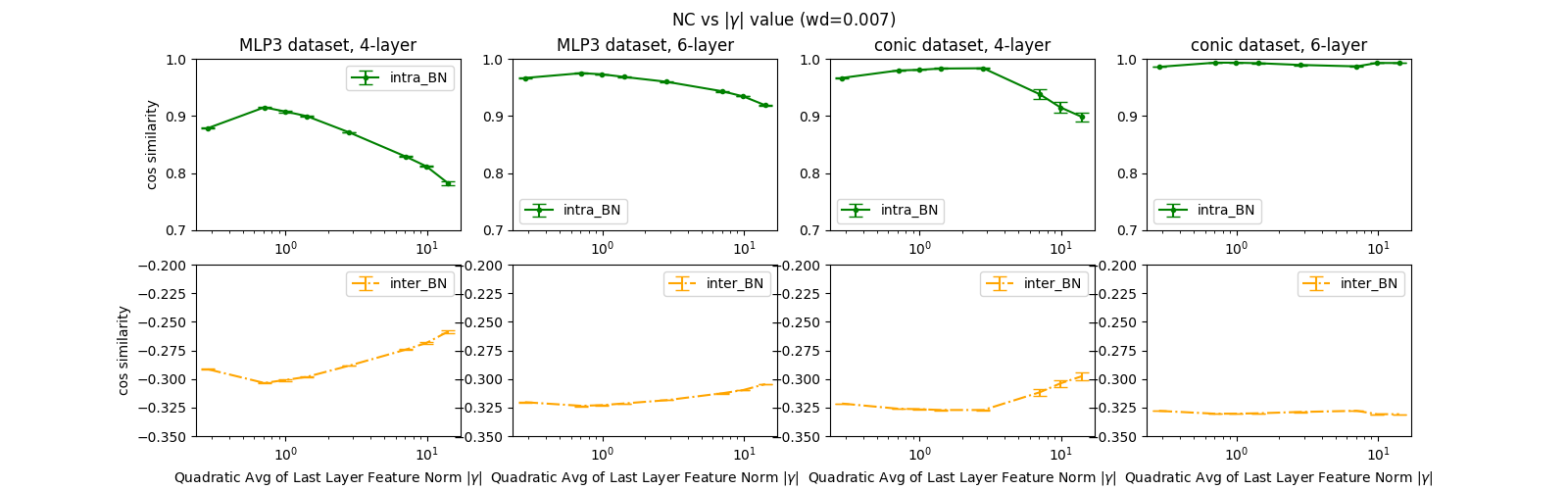}
    \caption{Relationship of \NC with last-layer feature norm under different WD values. Most experiments show that \NC is more significant at a higher last-layer feature norm. At very small feature norm and high weight decay, the model is no longer able to closely fit the training data, which explains a small initial decrease in \NC at the lower $\boldsymbol{\gamma}$ values}
    \label{fig:featnormwd}
\end{figure}

\newpage
\section{Proofs}
\label{appendix_proof}
\subsection{Proof of Lemma \ref{strongjensensub}}
Our first lemma demonstrate that if a set of variables achieves roughly equal value on the LHS and RHS of Jensen's inequality for a strongly convex function, then the mean of every subset cannot deviate too far from the global mean.
\begin{lemma}[Restatement of Lemma \ref{strongjensensub}]
    \label{strongjensensubappendix}
    Let $\{x_i\}_{i=1}^N\subset \mathcal{I}$ be a set of $N$ real numbers, let $\tilde{x}=\frac{1}{N}\sum_{i=1}^N x_i$ be the mean over all $x_i$ and $f$ be a function that is $m$-strongly-convex on $\mathcal{I}$. If
    $$\frac{1}{N}\sum_{i=1}^N f(x_i)\leq f(\tilde{x})+\epsilon$$
    Then for any subset of samples $S\subseteq [N]$, let $\delta=\frac{|S|}{N}$, there is $$\tilde{x}+\sqrt{\frac{2\epsilon(1-\delta)}{m\delta}}\geq \frac{1}{|S|}\sum_{i\in S} x_i\geq \tilde{x}-\sqrt{\frac{2\epsilon(1-\delta)}{m\delta}}$$
\end{lemma}
\begin{proof}
    For the proof, we use a result from~\cite{Merentes2010} which bounds the Jensen inequality gap using the variance of the variables for strongly convex functions:
    \begin{lemma}[Theorem 4 from~\cite{Merentes2010}]
        If $f: I\rightarrow \mathbb{R}$ is strongly convex with modulus $c$, then
        $$f\left(\sum_{i=1}^nt_ix_i\right)\leq \sum_{i=1}^n t_i f(x_i)-c\sum_{i=1}^n t_i(x_i-\bar{x})^2$$
        for all $x_1,\dots, x_n \in I$, $t_1,\dots,t_n>0$ with $t_1+\dots+t_n=1$ and $\bar{x}=t_1x_1+\dots+t_nx_n$
    \end{lemma}
    In the original definition of the authors, a strongly convex function with modulus $c$ is equivalent to a $2c$-strongly-convex function. We can apply $t_i=\frac{1}{N}$ for all $i$ and substitute the definition for strong convexity measure to obtain the following corollary:
    \begin{corollary}
    \label{jensengap}
    If $f: I\rightarrow \mathbb{R}$ is $m$-strongly-convex on $\mathcal{I}$, and 
    $$\frac{1}{N}\sum_{i=1}^Nf(x_i) = f\left(\frac{1}{N}\sum_{i=1}^N x_i\right)+\epsilon$$
    for $x_1,\dots, x_N \in \mathcal{I}$, then $\frac{1}{N}\sum_i(x_i-\bar{x})^2\leq \frac{2\epsilon}{m}$
    \end{corollary}
    From \ref{jensengap}, we know that $\frac{1}{N} \sum_{i=1}^n (x_i-\tilde{x})^2\leq \frac{2\epsilon}{m}$. Let $D=\sum_{i\in S}(x_i-\tilde{x})$, by the convexity of $x^2$, there is
    \begin{align*}
        \sum_{i=1}^n (x_i-\tilde{x})^2&=\sum_{i\in S}(x_i-\tilde{x})^2+\sum_{i\notin S}(x_i-\tilde{x})^2\\
        &\geq |S|(\frac{1}{|S|}\sum_{i\in S}(x_i-\tilde{x}))^2+ (N-|S|)(\frac{1}{N-|S|}\sum_{i\notin S}(x_i-\tilde{x}))^2\\
        &= \frac{1}{S}(\sum_{i\in S}(x_i-\tilde{x}))^2+\frac{1}{N-|S|}(\sum_{i\notin S}(x_i-\tilde{x}))^2\\
        &=\frac{1}{S}D^2+\frac{1}{N-|S|}(-D)^2\\
        &=\frac{D^2}{N}(\frac{1}{\delta}+\frac{1}{1-\delta})\\
        &=\frac{D^2}{N}(\frac{1}{\delta(1-\delta)})
    \end{align*}
    Therefore $\frac{D^2}{N}(\frac{1}{\delta(1-\delta)})\leq \frac{2\epsilon N}{m}$, and $|D|\leq \sqrt{\frac{2\epsilon\delta(1-\delta)N^2}{\lambda}}$. Using $\frac{1}{|S|}\sum_{i\in S} x_i=\frac{1}{|S|}(|S|\tilde{x}+D)$ and $|S|=\delta N$ completes the proof.
\end{proof}

\subsection{Proof of Theorem \ref{main}}

We first present several lemmas that facilitate the proof technique used in the main proof. Our first lemma in this section tighens Lemma \ref{jensengap} specifically for the function $e^x$ and only provides the upper bound. Note that, within any predefined range $[a,b]$, $\exp(x)$ can only be guaranteed to be $e^a$ strongly convex, which may be bad if the lower bound $a$ is small or does not exist. Our further result in the following lemma shows that we can provide a better upper bound of the subset mean for the exponential function that is dependent on $\exp(\tilde{x})$ and does not require other prior knowledge of the range of $x_i$:
\begin{lemma}
    \label{explemma}
    Let $\{x_i\}_{i=1}^N\subset \mathbb{R}$ be any set of $N$ real numbers, let $\tilde{x}=\frac{1}{N}\sum_{i=1}^N x_i$ be the mean over all $x_i$. If $$\frac{1}{N}\sum_{i=1}^N \exp(x_i)\leq \exp(\tilde{x})+\epsilon$$ then for any subset $S\subseteq [N]$, let $\delta=\frac{|S|}{N}$, the there is $$\frac{1}{|S|}\sum_{i\in S} x_i\leq \tilde{x}+\sqrt{\frac{2\epsilon}{\delta\exp(\tilde{x})}}.$$
\end{lemma}

\begin{proof}
Let $D=\sum_{i\in S} (x_i-\tilde{x})$. Note that if $D<0$ then the upper bound is obviously satisfied since the subset mean will be smaller than the global mean. Therefore, we only consider the case when $D>0$
\begin{align*}
    \sum_{i=1}^N \exp(x_i)&= \sum_{i\in S}\exp(x_i)+\sum_{i\notin S}\exp(x_i)\\
    &\geq |S|\exp(\frac{1}{|S|}\sum_{i\in S}x_i)+(N-|S|)\exp(\frac{1}{N-|S|}\sum_{i\notin S}x_i)\\
    &\geq |S|\exp(\tilde{x}+\frac{D}{|S|})+(N-|S|)\exp(\tilde{x}-\frac{D}{N-|S|})\\
    &\geq |S|\exp(\tilde{x})(1+\frac{D}{|S|}+\frac{D^2}{2|S|^2})+(N-|S|)\exp(\tilde{x})(1-\frac{D}{N-|S|})\\
    &=(N+\frac{D^2}{2|S|})\exp(\tilde{x})\\
    N\exp(\tilde{x})+N\epsilon&\geq (N+\frac{D^2}{2|S|})\exp(\tilde{x})\\
    D^2&\leq \frac{2|S|N\epsilon}{\exp(\tilde{x})}\\
    D&\leq N\sqrt{\frac{2\delta\epsilon}{\exp(\tilde{x})}}
\end{align*}
Using $\frac{1}{|S|}\sum_{i\in S} x_i=\frac{1}{|S|}(|S|\tilde{x}+D)$ and $|S|=\delta N$ completes the proof.
\end{proof}

Our next lemma focuses on a property of Batch Normalization: we show that BN effectively normalizes the quadratic average of the vector norms.
\begin{lemma}
\label{bnprop}
 Let $\{\mathbf{h}_i\}_{i=1}^N$ be a set of feature vectors immediately after Batch Normalization with variance vector $\boldsymbol{\gamma}$ and bias term $\boldsymbol{\beta}=0$ (i.e. $\mathbf{h}_i=BN(\mathbf{x}_i)$ for some $\{\mathbf{x}_i\}_{i=1}^N$). Then
$$\sqrt{\frac{1}{N}\sum_{i=1}^N \|\mathbf{h}_i\|_2^2}=\|\boldsymbol \gamma\|_2$$
\end{lemma}
\begin{proof}
    Let $\bg$ be the variance vector for the Batch Normalization layer, and consider a single batch $\{\x_i\}_{i=1}^B$ be a batch of $B$ vectors, and $$h_i^{(k)}=\frac{x^{(k)}_i-\tilde{x}^{(k)}}{\sigma^{(k)}}\times \gamma^{(k)}$$ for all $B$. By the linearity of mean and standard deviation, $\hat{x}^{(k)}_i=\frac{x^{(k)}_i-\tilde{x}^{(k)}}{\sigma^{(k)}_\x}$ must have mean 0 and standard deviation 1. As a result, $\sum_{i=1}^B \hat{x}^{(k)}_i=0$ and $\frac{1}{B}\sum_{i=1}^B (\hat{x}^{(k)}_i)^2=1$. Therefore, $$\sum_{i=1}^B (h_i^{(k)})^2=\sum_{i=1}^B \gamma^{(k)}(\hat{x}^{(k)}_i)^2=B(\gamma^{(k)})^2$$$$\sum_{i=1}^B \|\h_i\|^2=\sum_{k=1}^d\sum_{i=1}^B (h_i^{(k)})^2=\sum_{k=1}^d\sum_{i=1}^B \gamma^{(k)}(\hat{x}^{(k)}_i)^2=\sum_{k=1}^dB(\gamma^{(k)})^2=B\|\bg\|^2$$\\
    
Now, Consider a set of $N$ vectors divided into $m$ batches of size $\{B_j\}_{j=1}^m$. (This accounts for the fact that during training, the last mini-batch may have a different size than the other mini-batches if the number of training data is not a multiple of $B$). Then, $$\sum_{i=1}^N \|\h_i\|^2=\sum_{j=1}^m \sum_{i=1}^{B_j} \|\h_{j,i}\|^2=\sum_{j=1}^m B_j\|\bg\|^2=N\|\bg\|^2$$ Therefore, $\sqrt{\frac{1}{N}\sum_{i=1}^N \|\h_i\|^2}=\|\bg\|$
\end{proof}

Directly approaching the average intra-class and inter-class cosine similarity of vector set(s) is a relatively difficult task. Our following lemma shows that the inter-class and inter-class cosine similarities can be computed as the norm and dot product of the vectors $\tilde{\bar{\mathbf{h}}}_c$, respectively, where $\tilde{\bar{\mathbf{h}}}_c$ is the mean {\it normalized} vector among all vectors in a class.
\begin{lemma}
    \label{avgnormvec}
    Let $c,c'$ be 2 classes, each containing $N$ feature vectors $\h_{c,i}\in \mathbb{R}^d$. Define the average intra-class cosine similarity of picking two vectors from the same class $c$ as 
    $$\mathit{intra}_c=\frac{1}{N^2}\sum_{i=1}^N\sum_{j=1}^N \cos_{\angle}(\mathbf{h}_{c,i}, \mathbf{h}_{c,j})$$
    and the intra-class cosine similarity between two classes $c,c'$ is defined as the average cosine similarity of picking one feature vector of class $c$ and another from class $c'$ as
    $$\mathit{inter}_c=\frac{1}{N^2}\sum_{i=1}^N\sum_{j=1}^N \cos_{\angle}(\mathbf{h}_{c,i}, \mathbf{h}_{c',j})$$
    Let $\tilde{\bar{\mathbf{h}}}_c=\frac{1}{N}\sum_{i=1}^{N}\frac{\mathbf{h}_{c,i}}{\|\mathbf{h}_{c,i}\|}$. Then $\mathit{intra}_c=\|\tilde{\bar{\mathbf{h}}}_c\|^2$ and $\mathit{inter}_{c,c'}=\tilde{\bar{\mathbf{h}}}_c\cdot \tilde{\bar{\mathbf{h}}}_{c'}$
\end{lemma}

\begin{proof}
For the intra-class cosine similarity,
\begin{align*}
    \mathit{intra}_c &= \frac{1}{N^2}\sum_{i=1}^N\sum_{j=1}^N \bar{\mathbf{h}}_{c,i}\cdot \bar{\mathbf{h}}_{c,j}\\
    &= \frac{1}{N^2}\sum_{i=1}^N\sum_{j=1}^N \frac{\mathbf{h}_{c,i}}{\|\mathbf{h}_{c,i}\|}\cdot \frac{\mathbf{h}_{c,j}}{\|\mathbf{h}_{c,j}\|}\\
    &= \frac{1}{N^2}\sum_{i=1}^N\sum_{j=1}^N \frac{\mathbf{h}_{c,i}\cdot \mathbf{h}_{c,j}}{\|\mathbf{h}_{c,i}\|\|\mathbf{h}_{c,j}\|}\\
    &= \left(\frac{1}{N}\sum_{i=1}^N \frac{\mathbf{h}_{c,i}}{\|\mathbf{h}_{c,i}\|}\right)\cdot\left(\frac{1}{N}\sum_{j=1}^N \frac{\mathbf{h}_{c,j}}{\|\mathbf{h}_{c,j}\|}\right)\\
    &= \|\tilde{\bar{\mathbf{h}}}_c\|^2
\end{align*}

and for the inter-class cosine similarity,
\begin{align*}
    \mathit{inter}_{c,c'} &= \frac{1}{N^2}\sum_{i=1}^N\sum_{j=1}^N \bar{\mathbf{h}}_{c,i}\cdot \bar{\mathbf{h}}_{c',j}\\
    &= \frac{1}{N^2}\sum_{i=1}^N\sum_{j=1}^N \frac{\mathbf{h}_{c,i}}{\|\mathbf{h}_{c,i}\|}\cdot \frac{\mathbf{h}_{c',j}}{\|\mathbf{h}_{c',j}\|}\\
    &= \frac{1}{N^2}\sum_{i=1}^N\sum_{j=1}^N \frac{\mathbf{h}_{c,i}\cdot \mathbf{h}_{c',j}}{\|\mathbf{h}_{c,i}\|\|\mathbf{h}_{c',j}\|}\\
    &= \left(\frac{1}{N}\sum_{i=1}^N \frac{\mathbf{h}_{c,i}}{\|\mathbf{h}_{c,i}\|}\right)\cdot\left(\frac{1}{N}\sum_{j=1}^N \frac{\mathbf{h}_{c',j}}{\|\mathbf{h}_{c',j}\|}\right)\\
    &= \tilde{\bar{\mathbf{h}}}_c\cdot \tilde{\bar{\mathbf{h}}}_{c'}
\end{align*}
\end{proof}

We prove the intra-class cosine similarity by first showing that the norm of the mean (un-normalized) class-feature vector for a class is near the quadratic average of feature means (i.e., $\|\tilde{\h}_c\|=\|\frac{1}{N}\sum_{i=1}^N\h_{c,i}\|\approx \sqrt{\frac{1}{N}\sum_{i=1}^N\|\h_{c,i}\|^2}$). However, to show intra-class cosine similarity, we need instead a bound on $\|\tilde{\bar{\h}}_c\|=\|\frac{1}{N}\sum_{i=1}^N\bar{\h}_{c,i}\|$ (recall that $\bar{\rvv}=\frac{\rvv}{\|\rvv\|}$ denotes the normalized vector). The following lemma provides a conversion between these requirements:
\begin{lemma}
    \label{avgintralemma}
    Let unit vector $\mathbf{u} \in \mathbb{R}^d, \|\rvu\|=1$, and let $\{\mathbf{v}_i\}_{i=1}^N \subset \mathbb{R}^d$ be a set of vectors such that $\frac{1}{N}\sum_{i=1}^{N}\|\mathbf{v}_i\|^2 \leq \alpha^2$. Define the mean of the vectors $\mathbf{v}_i$ as $\tilde{\mathbf{v}} := \frac{1}{N}\sum_{i=1}^{N} \mathbf{v}_i$.
    
    Suppose that
    \[
    \langle \mathbf{u}, \tilde{\mathbf{v}} \rangle = \frac{1}{N} \sum_{i=1}^N \langle \mathbf{u}, \mathbf{v}_i \rangle \geq c,
    \]
    where $\frac{\alpha }{\sqrt{2}} \leq c \leq \alpha $. Define $\bar{\mathbf{v}}_i = \frac{\mathbf{v}_i}{\|\mathbf{v}_i\|}$ and $\tilde{\bar{\mathbf{v}}} := \frac{1}{N} \sum_{i=1}^N \bar{\mathbf{v}}_i$.

    Then,
    \[
    \|\tilde{\bar{\mathbf{v}}}\| \geq 2 \left(\frac{c}{\alpha }\right)^2 - 1.
    \]
\end{lemma}

The proof of Lemma~\ref{avgintralemma} uses a generalization of Holder's Inequality, which we state as follows.

\begin{lemma}[Generalized Holder's Inequality \cite{chen2014brief}]
\label{lemma_holder}
    For real positive exponents $\lambda_i$ satisfying $\lambda_a + \lambda_b + \cdots + \lambda_z = 1$, the following inequality holds.
    \[
    \sum_{i=1}^n |a_i|^{\lambda_a}|b_i|^{\lambda_b}\cdots |z_i|^{\lambda_z}
    \leq
    \left(\sum_{i=1}^n|a_i|\right)^{\lambda_a}\left(\sum_{i=1}^n|b_i|\right)^{\lambda_b}\cdots \left(\sum_{i=1}^n|z_i|\right)^{\lambda_z}
    \]
\end{lemma}
Now we are ready to prove Lemma~\ref{avgintralemma}.
\begin{proof}[Proof of Lemma~\ref{avgintralemma}]
    We divide all indices $i\in[N]$ into 2 sets: $$pos=\{i\in [N]|\langle\mathbf{u}, \mathbf{v}_i\rangle \geq 0\}$$ and $$neg=\{i\in [N]|\langle\mathbf{u}, \mathbf{v}_i\rangle < 0\}$$ 
    Denote $I = |pos|$ as the number of indices $i$ such that $\langle u,v_i\rangle \geq 0$. We assume wlog that $pos = \{1,2,\cdots,I\}$ and $neg = \{I+1,I+2,\cdots,N\}$.  Denote $a_i = \langle u,v_i\rangle$. Then we can decompose each vector $v_i$ as follows.
    \[
    v_i = a_i u + b_iw_i,\text{ where the unit vector }w_i\bot u, b_i\in\R
    \]
    Then by normalizing $v_i$ we get
    \[
    \bar{v_i} = \frac{a_i}{\sqrt{a_i^2 + b_i^2}}u + \frac{b_i}{\sqrt{a_i^2 + b_i^2}}w_i
    \]
    Thus we know the vector $\tilde{\bar{v}}$ can be represented as 
    \[
    \tilde{\bar{v}} = \frac{1}{N} \sum_{i=1}^N \frac{a_i}{\sqrt{a_i^2 + b_i^2}}u + 
    \frac{1}{N}\sum_{i=1}^N
    \frac{b_i}{\sqrt{a_i^2 + b_i^2}}w_i
    \]
    Its norm can be lower bounded by
    \begin{align*}
         \|\tilde{\bar{v}}\|^2 = 
         \left\lVert \frac{1}{N} \sum_{i=1}^N \frac{a_i}{\sqrt{a_i^2 + b_i^2}}u \right\lVert^2
        + \left\lVert  
        \frac{1}{N}\sum_{i=1}^N
        \frac{b_i}{\sqrt{a_i^2 + b_i^2}}w_i
        \right\lVert^2
        \geq \frac{1}{N^2}\left(\sum_{i=1}^N \frac{a_i}{\sqrt{a_i^2\ + b_i^2}}\right)^2 
    \end{align*}
    Take the square root of both side, and we get
    \begin{align}
    \label{eq:lemmac6_Nnormv}
        N\|\tilde{\bar{v}}\|\geq  \sum_{i=1}^N \frac{a_i}{\sqrt{a_i^2 + b_i^2}}
    =\sum_{i=1}^I \frac{a_i}{\sqrt{a_i^2 + b_i^2}} + \sum_{i=I+1}^N \frac{a_i}{\sqrt{a_i^2 + b_i^2}}
    \end{align}
    Since for any $i\geq I+1$, $a_i > 0$, and also we know for any $x,y\geq 0$,
    \[
    \left|\frac{x}{\sqrt{x^2+y^2}}\right|\leq 1
    \]
    Thus for any $i\geq I+1$,
    \[
    \frac{a_i}{\sqrt{a_i^2 + b_i^2}}\geq -1
    \]
    By substituting this into Equation~\ref{eq:lemmac6_Nnormv}, we have
    \begin{align}
    \label{eq:lemmac6_Nnormv2}
        N\|\tilde{\bar{v}}\|\geq  \sum_{i=1}^N \frac{a_i}{\sqrt{a_i^2 + b_i^2}}
    \geq \sum_{i=1}^I \frac{a_i}{\sqrt{a_i^2 + b_i^2}} -N+I
    \end{align}
    Since $\langle u, \tilde{v}\rangle \geq c$, we have
    \[
    \sum_{i=1}^N a_i \geq Nc
    \]
    Consequently,
    \begin{align}
    \label{eq:lemmac6_sum_ai}
        \sum_{i=1}^I a_i > \sum_{i=1}^N a_i \geq Nc
    \end{align}
    We also have $\frac{1}{N}\sum_{i=1}^N \|v_i\|^2\leq \alpha^2$. So we have
    \begin{align}
    \label{eq:lemmac6_sum_ai_sqaure}
        \sum_{i=1}^I a_i^2\leq 
        \sum_{i=1}^I(a_i^2 +b_i^2)\leq 
        \sum_{i=1}^N (a_i^2+b_i^2) \leq \sum_{i=1}^N \|v_i\|^2 \leq N\alpha^2
    \end{align}
    By Cauchy-Schwarz Inequality,
    \[
    \sum_{i=1}^I a_i^2 \sum_{i=1}^I 1
    \geq \left(\sum_{i=1}^I a_i\right)^2
    \]
    So we know
    \begin{align}
        \label{eq:lemmac6_I_lowerbound}
        I \geq \frac{\left(\sum_{i=1}^I a_i\right)^2}{\sum_{i=1}^I a_i^2}  \geq \frac{N^2c^2}{N\alpha^2} = \frac{Nc^2}{\alpha^2}
    \end{align}
    By Lemma~\ref{lemma_holder},
    \[
    \left(\sum_{i=1}^I \frac{a_i}{\sqrt{a_i^2+b_i^2}}\right)^{2/3}
    \left(\sum_{i=1}^I (a_i^2+b_i^2)\right)^{1/3}
    \geq \sum_{i=1}^I a_i ^{2/3}
    \]
    Combining with Equation~\ref{eq:lemmac6_sum_ai_sqaure}, we have
    \begin{align}
    \label{eq:lemmac6_holder1}
        \sum_{i=1}^I \frac{a_i}{\sqrt{a_i^2+b_i^2}}
        \geq 
        \sqrt{\frac{\left(\sum_{i=1}^I a_i^{2/3}\right)^3}{N\alpha^2}}
    \end{align}
    Then we apply Lemma~\ref{lemma_holder} again as follows.
    \[
    \left(\sum_{i=1}^I a_i^{2/3}\right)^{3/4}
    \left(\sum_{i=1}^I a_i^2\right)^{1/4}
    \geq \sum_{i=1}^I a_i
    \]
    Combining with Equation~\ref{eq:lemmac6_sum_ai} and Equation~\ref{eq:lemmac6_sum_ai_sqaure},
    \begin{align}
    \label{eq:lemmac6_holder2}
        \left(\sum_{i=1}^I a_i^{2/3}\right)^3
        \geq 
        \frac{(\sum_{i=1}^I a_i)^4}{\sum_{i=1}^I a_i^2}
        \geq
        \frac{N^4c^4}{N\alpha^2} = \frac{N^3c^4}{\alpha^2}
    \end{align}
    Using Equation~\ref{eq:lemmac6_holder1} and Equation~\ref{eq:lemmac6_holder2}, we have
    \[
    \sum_{i=1}^I \frac{a_i}{\sqrt{a_i^2+b_i^2}}
    \geq 
    \sqrt{\frac{N^2c^4}{\alpha^4}} = \frac{Nc^2}{\alpha^2}
    \]
    Plugging this into Equation~\ref{eq:lemmac6_Nnormv2} and apply Equation~\ref{eq:lemmac6_I_lowerbound}, we have
    \[  
    N\|\tilde{\bar{v}}\|
    \geq 
    \frac{Nc^2}{\alpha^2}-N+\frac{Nc^2}{\alpha^2} = N\left(\frac{2c^2}{\alpha^2}-1\right)
    \]
    This leads to our conclusion.
\end{proof}
To make this lemma generalize to other proofs in future work, we provide the generalized corollary of the above lemma by setting $\mathbf{u}$ to be the normalized mean vector of $\mathbf{v}$:
\begin{corollary}
        \label{intra_cor}
        Let $\{\mathbf{v}_i\}_{i=1}^N \subset \mathbb{R}^d$ such that $\frac{1}{N}\|\mathbf{v}_i\|^2\leq \alpha^2$. 
        If $$\|\tilde{\rvv}\|:=\|\frac{1}{N} \sum_{i=1}^N \mathbf{v}_i\|\geq c,$$ for $\frac{\alpha}{\sqrt{2}} \leq c \leq \alpha$ and let $\bar{\mathbf{v}}:=\frac{\mathbf{v}}{\|\mathbf{v}\|}$
        then $$\|\tilde{\bar{\mathbf{v}}}\|:=\|\frac{1}{N}\sum_{i=1}^N \bar{\mathbf{v}_i}\|\geq 2(\frac{c}{\alpha})^2-1.$$
\end{corollary}
\begin{proof}
    Let $\mathbf{u}:=\frac{\tilde{\mathbf{v}}}{\|\tilde{\mathbf{v}}\|}$ then $\|\rvu\|=1$, $$\frac{1}{N}\sum_{i=1}^N\langle\frac{\tilde{\mathbf{v}}}{\|\tilde{\mathbf{v}}\|}, \rvv_i\rangle=\langle\rvu,\tilde{\rvv}\rangle=\frac{\|\tilde{\rvv}\|^2}{\|\tilde{\rvv}\|}=\|\tilde{\rvv}\|\geq c$$
    The corollary directly follows from Lemma \ref{avgintralemma} with $\beta=\|\rvu\|=1$
\end{proof}

Similarly, for inter-class cosine similarity, we have the following lemma:
\begin{lemma}
    \label{interdivide}
    Let $\mathbf{w}\in \mathbb{R}^d$, $\{\mathbf{h}_{i}\}_{i=1}^N\subset \mathbb{R}^d$. Let $\tilde{\mathbf{h}}=\frac{1}{N}\sum_{i=1}^N \mathbf{h}_i$ and $\tilde{\bar{\mathbf{h}}}=\frac{1}{N}\sum_{i=1}^N \frac{\mathbf{h}_i}{\|\mathbf{h}_i\|}$. If the following condition is satisfied:
    \begin{align*}
    \mathbf{w}\cdot\tilde{\mathbf{h}}&= c & \text{for } c<0\\
    \|\mathbf{w}\|&\leq \beta\\
    \frac{1}{N}\sum_{i=1}^n \|\mathbf{h}_i\|^2 &\leq \alpha^2\\
    \|\tilde{\mathbf{h}}\| &\geq \alpha-\frac{\epsilon}{\beta}\\
    \epsilon&\ll \alpha\beta
    \end{align*}
    Then $\cos_\angle(\mathbf{w}, \tilde{\bar{\mathbf{h}}})\leq -\frac{c}{\alpha\beta}+4(\frac{\epsilon}{\alpha\beta})^{1/3}$
\end{lemma}

\begin{proof}
For $\mathbf{w} \in \mathbb{R}^d$, $\{\mathbf{h}_i\}_{i=1}^N\subset \mathbb{R}^d$

Let $a_i:=\frac{1}{N}\mathbf{w}\mathbf{h}_i$, $b_i:=\|\mathbf{h}_i\|$, $\epsilon':=\frac{\epsilon}{\beta}$, then the constraints of the above problem can be reformulated as follows:
\begin{align*}
    \max \sum_{i=1}^N\frac{a_i}{b_i}\\
    s.t. \sum_{i=1}^N a_i&\leq c\\
    \frac{1}{N}\sum_{i=1}^N b_i^2&=\alpha^2\\
    \frac{1}{N}\sum_{i=1}^N b_i&\geq \alpha - \epsilon'\\
    \forall i, |\frac{a_i}{b_i}|&\leq \beta.
\end{align*}
      Consider a random variable $B$ that uniformly picks a value from $\{b_i\}_{i=1}^N$. Then $\mathbb{E}[B]\geq \alpha-\frac{\epsilon}{\beta}$, $\mathbb{E}[B^2]= \alpha^2$, and therefore $\sigma_B=\sqrt{\mathbb{E}[B^2]-\mathbb{E}[B]^2}\leq \sqrt{2\alpha\epsilon}$. According to Chebyshev's inequality $$P(|B-(\alpha-\epsilon)|\geq k\sqrt{2\alpha\epsilon})\leq \frac{1}{k^2}.$$
    Note that for positive $a_i$, smaller $b_i$ means larger $\frac{a_i}{b_i}$ and for negative $a_i$, higher $b_i$ means larger $\frac{a_i}{b_i}$.
    Suppose that $\epsilon$ is sufficiently small such that $\epsilon \ll \sqrt{\epsilon}$.Therefore, an upper bound for $\frac{a_i}{b_i}$ when $a_i>0$ is 
    $$\frac{a_i}{b_i}\leq \begin{cases}\frac{a_i}{\alpha-k\sqrt{2\alpha\epsilon}} & b_i\geq \alpha-k\sqrt{2\alpha\epsilon}\\
    \beta & b_i< \alpha-k\sqrt{2\alpha\epsilon}\end{cases},$$
    and an upper bound for $a_i < 0$ would is
    $$\frac{a_i}{b_i}\leq \begin{cases}\frac{a_i}{\alpha+k\sqrt{2\alpha\epsilon}} & b_i\leq \alpha+k\sqrt{2\alpha\epsilon}\\
    0 & b_i>\alpha+k\sqrt{2\alpha\epsilon}\end{cases}.$$
    Suppose that $k\sqrt{\frac{2\epsilon}{\alpha}}$ is less than $\frac{1}{2}$, then
    $$\frac{a_i}{\alpha-k\sqrt{2\alpha\epsilon}}=\frac{a_i}{\alpha}\cdot \frac{1}{1-k\sqrt{\frac{2\epsilon}{\alpha}}}< \frac{a_i}{\alpha}\cdot (1+2k\sqrt{\frac{2\epsilon}{\alpha}})=\frac{a_i}{\alpha}+|\frac{a_i}{\alpha}|\cdot2k\sqrt{\frac{2\epsilon}{\alpha}}$$ when $a_i>0$, and similarly $$\frac{a_i}{\alpha+k\sqrt{2\alpha\epsilon}}=\frac{a_i}{\alpha}\cdot \frac{1}{1+k\sqrt{\frac{2\epsilon}{\alpha}}}<\frac{a_i}{\alpha}\cdot (1-2k\sqrt{\frac{2\epsilon}{\alpha}})=\frac{a_i}{\alpha}+|\frac{a_i}{\alpha}|\cdot2k\sqrt{\frac{2\epsilon}{\alpha}}$$ when $a_i<0$.
    Note that $$\sum_{i=1}^N |\frac{a_i}{\alpha}|\cdot2k\sqrt{\frac{2\epsilon}{\alpha}}\leq \sum_{i=1}^N\frac{\beta}{N}\cdot2k\sqrt{\frac{2\epsilon}{\alpha}}=2k\beta\sqrt{\frac{2\epsilon}{\alpha}}$$
    Therefore, an upper bound on the total sum would be:
    $$\frac{c}{\alpha}+2k\beta\sqrt{\frac{2\epsilon}{\alpha}}+\frac{\beta}{k^2}$$
    Set $k=(\sqrt{\frac{8\epsilon}{\alpha}})^{-\frac{1}{3}}$ to get:
    $$\frac{c}{\alpha}+2\beta(\sqrt{\frac{8\epsilon}{\alpha}})^{\frac{2}{3}}=\frac{c}{\alpha}+4\beta(\frac{\epsilon}{\alpha})^{\frac{1}{3}}$$
    Now, we substitute $\epsilon=\frac{\epsilon'}{\beta}$ we get:
    $\mathbf{w}\cdot \tilde{\bar{\mathbf{h}}}\leq \frac{c}{\alpha}+4\beta(\frac{\epsilon'}{\alpha\beta})^{1/3}$
    Since $|\mathbf{w}|\leq \beta$ and $|\tilde{\bar{\mathbf{h}}}|\leq 1$, we get that
    $$\cos_\angle(\mathbf{w},\tilde{\bar{\mathbf{h}}})\leq \frac{c}{\alpha\beta}+4(\frac{\epsilon'}{\alpha\beta})^{1/3}$$
\end{proof}
\begin{theorem}[Detailed version of Theorem 2.1]
    \label{main2}
    For any neural network classifier without bias terms trained on dataset with the number of classes $C\geq 3$ and samples per class $N\geq 1$, under the following assumptions:
    \begin{enumerate}
        \item The quadratic average of the feature norms $\sqrt{\frac{1}{CN} \sum_{c=1}^C\sum_{i=1}^N \|\mathbf{h}_{c,i}\|^2}\leq \alpha$
        \item The Frobenius norm of the last-layer weight $\|\mathbf{W}\|_F\leq \sqrt{C}\beta$
        \item The average cross-entropy loss over all samples $\mathcal{L}\leq m+\epsilon$ for small $\epsilon$
    \end{enumerate}
    where $m=\log(1+(C-1)\exp(-\frac{C}{C-1}\alpha\beta))$ is the minimum achievable loss for any set of weight and feature vectors satisfying the norm constraints,
    then for at least $1-\delta$ fraction of all classes , with $\frac{\epsilon}{\delta}\ll1$, for small constant $\kappa>0$ there is
    $$\mathit{intra}_c\geq 1-\frac{C-1}{C\alpha\beta}\sqrt{\frac{128\epsilon(1-\delta)\exp(\kappa C\alpha\beta)}{\delta}}=1-O(\frac{e^{O(C\alpha\beta)}}{\alpha\beta}\sqrt{\frac{\epsilon}{\delta}}),$$
    and also for a cosine similarity representation of NC3 in \cite{doi:10.1073/pnas.2015509117}:
    $$\cos_{\angle}(\dot{\mathbf{w}}_c, \tilde{\mathbf{h}_c})\geq 1-2\sqrt{\frac{2\epsilon(1-\delta)e^{\kappa C\alpha\beta}}{\delta}}=1-O(e^{O(C\alpha\beta)}\sqrt{\frac{\epsilon}{\delta}}),$$
    and for at least $1-\delta$ fraction of all pairs of classes $c,c'$, with $\frac{\epsilon}{\delta}\ll1$, there is
    \begin{align*}
    inter_{c,c'}&\leq -\frac{1}{C-1}+\frac{C}{C-1}\frac{\exp(\kappa C\alpha\beta)}{\alpha\beta}\sqrt{\frac{2\epsilon}{\delta}}+4(\frac{2\exp(\kappa C\alpha\beta)}{\alpha\beta}\sqrt{\frac{2\epsilon}{\delta}})^{1/3}+\sqrt{\frac{\exp(\kappa C\alpha\beta)}{\alpha\beta}\sqrt{\frac{2\epsilon}{\delta}}}\\
&= -\frac{1}{C-1} +O(\frac{e^{O(C\alpha\beta)}}{\alpha\beta}(\frac{\epsilon}{\delta})^{1/6})
\end{align*}
\end{theorem}
\begin{proof}
Recall the definition of $\mathcal{L}$:
$$\mathcal{L}=\frac{1}{CN}\sum_{c=1}^C \sum_{i=1}^{N} \mathcal{L}_{\mathrm{CE}}\left(f(\boldsymbol{x}_{c,i};\boldsymbol{\theta}), \boldsymbol{y}_c\right)=\frac{1}{CN}\sum_{c=1}^C \sum_{i=1}^{N} \mathcal{L}_{\mathrm{CE}}\left(\boldsymbol{Wh}_{c,i}, \boldsymbol{y}_c\right),$$
Let $$L_{c,i}:=\mathcal{L}_{\mathrm{CE}}\left(\boldsymbol{Wh}_{c,i}, \boldsymbol{y}_c\right)$$ denote the individual loss for sample $i$ from class $c$.

First, consider the minimum achievable average loss for a single class $c$:
\begin{align*}
    \frac{1}{N}\sum_{i=1}^N L_{c,i} &= \frac{1}{N}\sum_{i=1}^N CE_c(\mathbf{Wh}_{c,i})\\
    &\geq CE_c(\frac{1}{N}\sum_{i=1}^N \mathbf{Wh}_{c,i})_c\\
    &=\log\left(1+\sum_{c'\neq c}\exp(\frac{1}{N} \sum_{i=1}^{N}(\mathbf{w}_{c'}-\mathbf{w}_{c})\mathbf{h}_{c, i})\right)\\
    &=\log\left(1+\sum_{c'\neq c}\exp((\mathbf{w}_{c'}-\mathbf{w}_{c})\tilde{\mathbf{h}}_{c})\right)\\
    &\geq \log\left(1+(C-1)\exp(\frac{1}{(C-1)}(\sum_{c'=1}^{C}\mathbf{w}_{c'}\mathbf{\tilde{h}}_{c}-C\mathbf{w}_{c}\mathbf{\tilde{h}}_{c}))\right)\\
    &= \log\left(1+(C-1)\exp(\frac{1}{(C-1)}(\sum_{c'=1}^{C}\mathbf{w}_{c'}-C\mathbf{w}_{c})\mathbf{\tilde{h}}_{c})\right)\\
    &= \log\left(1+(C-1)\exp(\frac{C}{C-1}(\tilde{\mathbf{w}}-\mathbf{w}_{c})\mathbf{\tilde{h}}_{c})\right)\\
    &= \log\left(1+(C-1)\exp(-\frac{C}{C-1}\dot{\w}_c\mathbf{\tilde{h}}_{c})\right)
\end{align*}
Where we define $\dot{\w}_c:=\mathbf{w}_{c}-\tilde{\mathbf{w}}$ Let $\overrightarrow{\mathbf{w}}=[\mathbf{w}_1-\tilde{\mathbf{w}}, \mathbf{w}_2-\tilde{\mathbf{w}}, \dots, \mathbf{w}_{C}-\tilde{\mathbf{w}}]=[\dot{\w}_1, \dot{\w}_2, \dots, \dot{\w}_{C}]$, and $\overrightarrow{\mathbf{h}}=[\tilde{\mathbf{h}}_1, \tilde{\mathbf{h}}_2, \dots, \tilde{\mathbf{h}}_c]\in \mathbf{R}^{Cd}$. Note that
\begin{align*}
    \|\overrightarrow{\mathbf{w}}\|^2&=\sum_{c=1}^C\|\mathbf{w}_c-\tilde{\mathbf{w}}\|^2=\sum_{c=1}^C\left(\|\mathbf{w}_c\|^2-2\mathbf{w}_c\tilde{\mathbf{w}}+\|\tilde{\mathbf{w}}\|^2\right)\\
    &=\sum_{c=1}^C\|\mathbf{w}_c\|^2-C\|\tilde{\mathbf{w}}\|^2\leq \sum_{c=1}^C\|\mathbf{w}_c\|^2= \|\mathbf{W}\|_F^2\leq C\beta^2
\end{align*}
and also
\begin{align*}
    \|\overrightarrow{\mathbf{h}}\|^2&=\sum_{c=1}^C \|\tilde{\mathbf{h}}_c\|^2= \sum_{c=1}^C \|\frac{1}{N}\sum_{i=1}^N\mathbf{h}_{c,i}\|^2\leq \sum_{c=1}^C \left(\frac{1}{N}\sum_{i=1}^N\|\mathbf{h}_{c,i}\|\right)^2\\
    &\leq \frac{1}{N}\sum_{c=1}^C \sum_{i=1}^N\|\mathbf{h}_{c,i}\|^2=C\alpha^2\\
\end{align*}
The first inequality uses the triangle inequality and the second uses $\mathbb{E}[X^2]\geq \mathbb{E}[X]^2$
Now consider the total average loss over all classes:
\begin{align*}
    \mathcal{L}&=\frac{1}{CN}\sum_{c=1}^C\sum_{i=1}^N L_{c,i}\\
    &\geq \frac{1}{C}\sum_{c=1}^C \log\left(1+(C-1)\exp(\frac{C}{C-1}(\tilde{\mathbf{w}}-\mathbf{w}_{c})\mathbf{\tilde{h}}_{c})\right)\\
    &\geq \log\left(1+(C-1)\exp(\frac{C}{C-1}\cdot\frac{1}{C}\sum_{c=1}^C (\tilde{\mathbf{w}}-\mathbf{w}_{c})\mathbf{\tilde{h}}_{c})\right)& \text{Jensen's}\\ 
    &\geq \log\left(1+(C-1)\exp(-\frac{1}{C-1}\overrightarrow{\mathbf{w}}\cdot\overrightarrow{\mathbf{h}})\right)\\
    &\geq \log\left(1+(C-1)\exp(-\frac{C}{C-1}\alpha\beta\right)\\
    &=m,
\end{align*}
showing that $m$ is indeed the minimum achievable average loss among all samples.\\
Now we instead consider when the final average loss is near-optimal of value $m+\epsilon$ with $\epsilon \ll 1$. We use a new $\epsilon$ to represent the gap introduced by each inequality in the above proof. Additionally, since the average loss is near-optimal, there must be $\dot{\mathbf{w}}_c\mathbf{\tilde{h}}_{c}\geq 0$ for any sufficiently small $\epsilon$:
\begin{align}
    \frac{1}{N}\sum_{i=1}^N L_{c,i} &= \frac{1}{N}\sum_{i=1}^N \text{softmax}(\mathbf{Wh}_{c,i})_c\\
    &\geq \text{softmax}(\frac{1}{N}\sum_{i=1}^N \mathbf{Wh}_{c,i})_c\\
    &= \log\left(1+\sum_{c'\neq c}\exp(\frac{1}{N} \sum_{i=1}^{N}\mathbf{w}_{c'}\mathbf{h}_{c, i}-\frac{1}{N} \sum_{i=1}^{N}\mathbf{w}_{c}\mathbf{h}_{c, i})\right)\\
    &=\log\left(1+\sum_{c'\neq c}\exp(\frac{1}{N} \sum_{i=1}^{N}(\mathbf{w}_{c'}-\mathbf{w}_{c})\mathbf{h}_{c, i})\right)\\
    &=\log\left(1+\sum_{c'\neq c}\exp((\mathbf{w}_{c'}-\mathbf{w}_{c})\tilde{\mathbf{h}}_{c})\right)\\
    &= \log\left(1+(C-1)\exp(\frac{1}{(C-1)}(\sum_{c'=1}^{C}\mathbf{w}_{c'}\mathbf{\tilde{h}}_{c}-C\mathbf{w}_{c}\mathbf{\tilde{h}}_{c}))+\epsilon_{1,c}'\right)\\
    &= \log\left(1+(C-1)\exp(\frac{1}{(C-1)}(\sum_{c'=1}^{C}(\mathbf{w}_{c'}-\mathbf{w}_{c})\mathbf{\tilde{h}}_{c})+\epsilon_{1,c}'\right)\\
    &= \log\left(1+(C-1)\exp(\frac{C}{C-1}(\tilde{\mathbf{w}}-\mathbf{w}_{c})\mathbf{\tilde{h}}_{c})+\epsilon_{1,c}'\right)\\
    & \geq \log\left(1+(C-1)\exp(-\frac{C}{C-1}\dot{\mathbf{w}}_c\mathbf{\tilde{h}}_{c})\right)+\frac{\epsilon_{1,c}'}{1+(C-1)\exp(-\frac{C}{C-1}\dot{\mathbf{w}}_c\mathbf{\tilde{h}}_{c})}\\
    &\geq \log\left(1+(C-1)\exp(-\frac{C}{C-1}\dot{\mathbf{w}}_c\mathbf{\tilde{h}}_{c})\right)+\frac{\epsilon_{1,c}'}{C}
\end{align}
where $\epsilon_{1,c}':=\exp(\frac{1}{(C-1)}(\sum_{c'=1}^{C}\mathbf{w}_{c'}\mathbf{\tilde{h}}_{c}-C\mathbf{w}_{c}\mathbf{\tilde{h}}_{c}))-\sum_{c'\neq c}\exp((\mathbf{w}_{c'}-\mathbf{w}_{c})\tilde{\mathbf{h}}_{c})$
and also
\begin{align}
    \mathcal{L}&=\frac{1}{CN}\sum_{c=1}^C\sum_{i=1}^N L_{c,i}\\
    &\geq \frac{1}{C}\sum_{c=1}^C \left(\log\left(1+(C-1)\exp(-\frac{C}{C-1}\dot{\mathbf{w}}_c\mathbf{\tilde{h}}_c)\right)+\frac{\epsilon_{1,c}'}{C}\right)\\
    &= \log\left(1+(C-1)\exp(-\frac{C}{C-1}\cdot\frac{1}{C}\sum_{c=1}^C \dot{\mathbf{w}}_c\mathbf{\tilde{h}}_c)\right)+\frac{1}{C}\sum_{c=1}^C\frac{\epsilon_{1,c}'}{C}+\epsilon_2'& \text{Jensen's with gap } \epsilon_2' \label{eps2}\\ 
    &= \log\left(1+(C-1)\exp(-\frac{1}{C-1}\overrightarrow{\mathbf{w}}\cdot\overrightarrow{\mathbf{h}})\right)+\frac{1}{C}\sum_{c=1}^C\frac{\epsilon_{1,c}'}{C}+\epsilon_2'\\
    &= \log\left(1+(C-1)\exp(-\frac{C}{C-1}\alpha\beta+\epsilon_3')\right)+\frac{1}{C}\sum_{c=1}^C\frac{\epsilon_{1,c}'}{C}+\epsilon_2'
\end{align}
where $$\epsilon_2':=\frac{1}{C}\sum_{c=1}^C\log\left(1+(C-1)\exp(-\frac{C}{C-1}\dot{\mathbf{w}}_c\mathbf{\tilde{h}}_c)\right)-\log\left(1+(C-1)\exp(-\frac{C}{C-1}\cdot\frac{1}{C}\sum_{c=1}^C \dot{\mathbf{w}}_c\mathbf{\tilde{h}}_c)\right)$$ and $\epsilon_3':=\frac{1}{C-1}(C\alpha\beta-\overrightarrow{\mathbf{w}}\cdot\overrightarrow{\mathbf{h}})$

Consider $\log(1+(C-1)\exp(-\frac{C\alpha\beta}{C-1}+\epsilon'_3))$:
Let $\gamma' = (C-1)\exp(-\frac{C\alpha\beta}{C-1})$\\
\begin{align*}
    \log(1+(C-1)\exp(-\frac{C\alpha\beta}{C-1}+\epsilon'_3))
    &=\log(1+(C-1)\exp(-\frac{C\alpha\beta}{C-1})\exp(\epsilon_3'))\\
    &=\log(1+(C-1)\exp(-\frac{C\alpha\beta}{C-1})\exp(\epsilon_3'))\\
    &=\log(1+\gamma'\exp(\epsilon_3'))\\
    &\geq\log(1+\gamma'(1+\epsilon_3'))\\
    &=\log(1+\gamma'+\gamma'\epsilon_3')\\
    &\geq \log(1+\gamma')+\frac{\gamma'\epsilon_3'}{1+\gamma'+\gamma'\epsilon_3'}
\end{align*}
Since $m+\epsilon=\log(1+\gamma')+\epsilon\geq \log(1+(C-1)\exp(-\frac{C\alpha\beta}{C-1}+\epsilon'_3))$, we get that $\epsilon \geq \frac{\gamma'\epsilon_3'}{1+\gamma'+\gamma'\epsilon_3'}$, and $$\epsilon_3' \leq \frac{\epsilon (1 + \gamma')}{\gamma' (1 - \epsilon)}=\frac{\epsilon}{1 - \epsilon}\cdot \frac{1 + \gamma'}{\gamma'}$$
for $\epsilon < 1$. 
By definition of $\epsilon'_3$, we know that
$$\overrightarrow{\mathbf{w}}\cdot\overrightarrow{\mathbf{h}}=\sum_{c=1}^C\dot{\mathbf{w}}_c\mathbf{\tilde{h}}_c \geq C\alpha\beta - (C-1)\cdot \frac{\epsilon}{1 - \epsilon}\cdot \frac{1 + \gamma'}{\gamma'}=C\alpha\beta - \frac{\epsilon}{1 - \epsilon}\cdot[\exp(\frac{C\alpha\beta}{C-1})+C-1]$$
For simplicity, let $\delta_2=\frac{\epsilon}{1 - \epsilon}\cdot[\exp(\frac{C\alpha\beta}{C-1})+C-1]$

Since $\|\overrightarrow{\mathbf{w}}\|\leq \sqrt{C}\beta$, we know that $\|\overrightarrow{\mathbf{h}}\|\geq \sqrt{C}\alpha-\frac{\delta_2}{\sqrt{C}\beta}$ and 
$$\|\overrightarrow{\mathbf{h}}\|^2=\sum_{c=1}^C\|\tilde{\rvh}_c\|^2\geq C\alpha^2-2\frac{\delta_2\alpha}{\beta}.$$

By Corollary \ref{intra_cor} we know that:
$$\|\tilde{\bar{\mathbf{h}}}_c\|\geq 2\left(\frac{\|\tilde{\rvh}_c\|}{\sqrt{\frac{1}{N}\sum_{i=1}^N\|\rvh_{c,i}\|}}\right)^2-1.$$

Let $\alpha_c:=\sqrt{\frac{1}{N}\sum_{i=1}^N\|\rvh_{c,i}\|^2}$, then $\sum_{c=1}^C\alpha_c^2\leq C\alpha^2$.

Now, using the bound on $\|\tilde{\bar{\mathbf{h}}}_c\|$ and the definition of $\alpha_c$, we can write:
$$\|\tilde{\bar{\mathbf{h}}}_c\| \geq 2\left(\frac{\|\tilde{\rvh}_c\|}{\alpha_c}\right)^2-1.$$

Summing over all classes $c = 1, \ldots, C$, we get:
$$\sum_{c=1}^C \|\tilde{\bar{\mathbf{h}}}_c\| \geq \sum_{c=1}^C \left(2\left(\frac{\|\tilde{\rvh}_c\|}{\alpha_c}\right)^2-1\right).$$

Since $\alpha_c = \sqrt{\frac{1}{N}\sum_{i=1}^N\|\rvh_{c,i}\|^2}$, we know that:
$$\sum_{c=1}^C \alpha_c^2 \leq C\alpha^2.$$


\begin{proposition}
    Given $\{a_i\}_{i=1}^N$ and $\{b_i\}_{i=1}^N$ such that $a_i\geq 0$ and $b_i> 0$ for all $i$, then $\sum_{i=1}^N \frac{a_i}{b_i}\geq N\frac{\sum_{i=1}^n a_i}{\sum_{i=1}^n b_i}$
\end{proposition}

Hence,
$$\sum_{c=1}^C \|\tilde{\bar{\mathbf{h}}}_c\| \geq 2\sum_{c=1}^C \left(\frac{\|\tilde{\rvh}_c\|^2}{\alpha_c^2}\right) - C\geq 2C \left(\frac{\sum_{c=1}^C\|\tilde{\rvh}_c\|^2}{\sum_{c=1}^C\alpha_c^2}\right) - C.$$

Using the bound $\sum_{c=1}^C \|\tilde{\rvh}_c\|^2 \geq C\alpha^2 - 2\frac{\delta_2\alpha}{\beta}$, we obtain:
$$\sum_{c=1}^C \|\tilde{\bar{\mathbf{h}}}_c\| \geq 2C\left(\frac{C\alpha^2 - 2\frac{\delta_2\alpha}{\beta}}{\sum_{c=1}^C \alpha_c^2}\right) - C.$$

Since $\sum_{c=1}^C \alpha_c^2 \leq C\alpha^2$, we can write:
$$\sum_{c=1}^C \|\tilde{\bar{\mathbf{h}}}_c\| \geq 2C\left(\frac{C\alpha^2 - 2\frac{\delta_2\alpha}{\beta}}{C\alpha^2}\right) - C.$$

Simplifying the expression:
$$\sum_{c=1}^C \|\tilde{\bar{\mathbf{h}}}_c\| \geq 2C\left(1 - \frac{2\frac{\delta_2\alpha}{\beta}}{C\alpha^2}\right) - C.$$

Further simplifying:
$$\sum_{c=1}^C \|\tilde{\bar{\mathbf{h}}}_c\| \geq 1 - \frac{4\delta_2}{\alpha\beta}=C-\frac{4\epsilon}{1 - \epsilon}\cdot\frac{\exp(\frac{C\alpha\beta}{C-1})+C-1}{\alpha\beta}$$

Since each $\|\tilde{\bar{\mathbf{h}}}_c\|\leq 1, \forall c$, we can use Markov's inequality to get that there are at least $1-\delta$ fraction of classes for which:
$$intra_c=\|\tilde{\bar{\mathbf{h}}}_c\|\geq 1-\frac{4\epsilon}{(1 - \epsilon)\delta}\cdot\frac{\exp(\frac{C\alpha\beta}{C-1})+C-1}{C\alpha\beta}=1-O\left(\frac{\epsilon}{\delta}\exp\left(\alpha\beta(1+o(C))\right)\right)$$

Thus using the fact that $1+(C-1)\exp(-\frac{C\alpha\beta}{C-1})\leq C$
\begin{align*}
    \mathcal{L} &\geq \log(1+(C-1)\exp(-\frac{C\alpha\beta}{C-1}))+\frac{1}{C}\sum_{c=1}^C\frac{\epsilon_{1,c}'}{C}+\epsilon_2'+\frac{\gamma'}{1+\gamma'}\epsilon_3'\\
    \epsilon &\geq \frac{1}{C}\sum_{c=1}^C\frac{\epsilon_{1,c}'}{C}+\epsilon_2'+\frac{\gamma'}{1+\gamma'}\epsilon_3'
\end{align*}
Note that while we do not know how $\epsilon$ is distributed among the different gaps, all the bounds involving $\epsilon_{1,c}', \epsilon_2', \epsilon_3'$ always hold in the worst case scenario subject to the constraint $\epsilon \geq \frac{1}{C}\sum_{c=1}^C\frac{\epsilon_{1,c}'}{C}+\epsilon_2'+\frac{\gamma'}{1+\gamma'}\epsilon_3'$. Note that $\|\tilde{\mathbf{h}_c}\|\leq \sum_{c'=1}^C\|\tilde{\mathbf{h}_{c'}}\|\leq \sqrt{C}\alpha$, and $\|\dot{\mathbf{w}_c}\|\leq \|\boldsymbol{W}\|_F=\sqrt{C}\beta$ therefore $\dot{\mathbf{w}}_c\tilde{\mathbf{h}}_{c} \geq -C\alpha\beta$.
We also know that 
$$\frac{1}{C}\sum_{c=1}^C \dot{\mathbf{w}}_c\tilde{\mathbf{h}}_{c}= \frac{1}{C} \overrightarrow{\mathbf{w}}\cdot\overrightarrow{\mathbf{h}}=\frac{1}{C}(C\alpha\beta-(C-1)\epsilon_3')=\alpha\beta-\frac{C-1}{C}\epsilon_3'$$

We now focus on the implication of $\epsilon_2'$ from (\ref{eps2}). Note that the relaxation can be written as
$$\frac{1}{C}\sum_{i=1}^C\log\left(1+(C-1)\exp(x_c))\right)=\log\left(1+(C-1)\exp\left(\frac{1}{C}\sum_{i=1}^C x_c\right)\right)+\epsilon_2'$$
with $x_c=-\frac{C}{C-1}(\dot{\mathbf{w}}_c\mathbf{\tilde{h}}_c)$ and $\epsilon_2'\geq 0$ because of the strong convexity of $\log(1+(C-1)\exp(x))$. Therefore, in order to apply Lemma \ref{strongjensensub}, we would first need to determine the degree of strong convexity of $\log(1+(C-1)\exp(x))$. Note that a function is $\lambda$ strongly convex if its second-order derivative is always at least $\lambda$.\\

The second-order derivative of $\log(1+(C-1)\exp(x))$ is 
$$\frac{(C-1)\exp(x)}{(1+(C-1)\exp(x))^2}=1/((C-1)\exp(x)+2+\frac{1}{(C-1)\exp(x)}),$$
which is $e^{-\kappa C\alpha\beta}$ for any $x\in[-\frac{C^2}{C-1}\alpha\beta, \frac{C^2}{C-1}\alpha\beta]$ for small constant $\kappa$, we denote as $O(C\alpha\beta)$ further. Therefore, the function $\log(1+(C-1)\exp(x))$ is $\lambda$-strongly-convex for $\lambda=e^{-O(C\alpha\beta)}$ Thus, for any subset $S\subseteq [C]$, let $\delta=\frac{|S|}{C}$, by Lemma \ref{strongjensensub}:
\begin{align*}
    -\frac{C}{C-1} \sum_{c\in S} \dot{\mathbf{w}}_c\mathbf{\tilde{h}}_c &\leq \delta C(-\frac{1}{C-1}\overrightarrow{\mathbf{w}}\cdot\overrightarrow{\mathbf{h}})+C\sqrt{\frac{2\epsilon_2' \delta(1-\delta)}{\lambda}} \\
    \sum_{c\in S} \dot{\mathbf{w}}_c\mathbf{\tilde{h}}_c &\geq \delta \overrightarrow{\mathbf{w}}\cdot\overrightarrow{\mathbf{h}}-(C-1)\sqrt{\frac{2\epsilon_2' \delta(1-\delta)}{\lambda}}\\
    \sum_{c\in S} \alpha_c\beta_c&=\sum_{c\in [C]}\alpha_c\beta_c-\sum_{c\notin S}\alpha_c\beta_c\\
    &\leq \sum_{c\in [C]}\alpha_c\beta_c-\sum_{c\notin S} \dot{\mathbf{w}}_c\mathbf{\tilde{h}}_c\\
    &\leq C\alpha\beta-\sum_{c\notin [C]-S}\dot{\mathbf{w}}_c\mathbf{\tilde{h}}_c\\
    &\leq C\alpha\beta-(1-\delta)\overrightarrow{\mathbf{w}}\cdot\overrightarrow{\mathbf{h}}+(C-1)\sqrt{\frac{2\epsilon_2' \delta(1-\delta)}{\lambda}}
\end{align*}
Let $\alpha_c=\sqrt{\frac{1}{N}\sum_{i=1}^{N}\|\mathbf{h}_{c,i}\|^2}$ and $\beta_c=\|\dot{\w}_c\|$. Note that since $-\frac{1}{C-1}\overrightarrow{\mathbf{w}}\cdot\overrightarrow{\mathbf{h}}=-\frac{C}{C-1}\alpha\beta+\epsilon_3'$, there is $\overrightarrow{\mathbf{w}}\cdot\overrightarrow{\mathbf{h}}=C\alpha\beta-(C-1)\epsilon_3'$. Therefore,
\begin{align*}
\sum_{c\in S} \dot{\mathbf{w}}_c\mathbf{\tilde{h}}_c &\geq \delta C\alpha\beta-\delta(C-1)\epsilon_3'-(C-1)\sqrt{\frac{2\epsilon_2' \delta(1-\delta)}{\lambda}}\\
\sum_{c\in S} \alpha_c\beta_c&\leq \delta C\alpha \beta+(1-\delta)(C-1)\epsilon'_3+(C-1)\sqrt{\frac{2\epsilon_2' \delta(1-\delta)}{\lambda}}
\end{align*}
Therefore, there are at most $\delta C$ classes for which
\begin{align}
   \dot{\mathbf{w}}_c\mathbf{\tilde{h}}_c &\leq \alpha\beta-\frac{(C-1)}{C}\epsilon_3'-\frac{C-1}{C}\sqrt{\frac{2\epsilon_2' (1-\delta)}{\delta\lambda}}\label{largewh}
\end{align}
and also there are at most $\delta C$ classes for which
\begin{align}
    \alpha_c\beta_c&\geq \alpha \beta+\frac{(1-\delta)(C-1)}{\delta C}\epsilon'_3+\frac{C-1}{C}\sqrt{\frac{2\epsilon_2'(1-\delta)}{\delta\lambda}}\label{smallab}
\end{align}
Thus, for at least \((1-2\delta)C\) classes, we have
\begin{align}
    \frac{\dot{\mathbf{w}}_c \tilde{\mathbf{h}}_c}{\alpha_c \beta_c} \geq 1 - \left(\frac{C-1}{C\alpha\beta}\right) \left(\frac{\epsilon'_3}{\delta} - 2\sqrt{\frac{2\epsilon'_2 (1-\delta)}{\delta\lambda}}\right) \label{divbound}
\end{align}

By setting \(\epsilon'_2 = \epsilon\) and \(\epsilon'_3 = 0\), we obtain the following upper bound on the cosine of the angle between \(\dot{\mathbf{w}}_c\) and \(\tilde{\mathbf{h}}_c\):
\[
\cos(\angle(\dot{\mathbf{w}}_c, \tilde{\mathbf{h}}_c)) \geq 1 - 2\sqrt{\frac{2\epsilon (1-\delta)}{\delta\lambda}}
\]

Using \(\lambda = e^{-O(C\alpha\beta)}\), we get the NC3 bound in the theorem:
\[
\cos(\angle(\dot{\mathbf{w}}_c, \tilde{\mathbf{h}}_c)) \geq 1 - 2\sqrt{\frac{2\epsilon (1-\delta)e^{O(C\alpha\beta)}}{\delta}} = 1 - O\left(e^{O(C\alpha\beta)} \sqrt{\frac{\epsilon}{\delta}}\right)
\]

Let \(\mathcal{C}\) denote the set of classes for which the above inequality holds. By applying Lemma \ref{avgintralemma} to the set of vectors \(\{\mathbf{h}_{c,i}\}\) where \(\mathbf{v}_i = \mathbf{h}_{c,i}\), \(\mathbf{u} = \frac{\dot{\mathbf{w}}_c}{\|\dot{\mathbf{w}}_c\|}\), and \(\beta = 1\), and using lemma \ref{avgintralemma} that $\mathit{intra}_c=\|\tilde{\bar}\rvh_c\|$ we obtain
\[
\mathit{intra}_c=\|\tilde{\bar}\rvh_c\|  \geq 1 - 4 \left(\frac{C-1}{C\alpha\beta}\right) \left(\frac{\epsilon'_3}{\delta} - 2\sqrt{\frac{2\epsilon'_2 (1-\delta)}{\delta\lambda}}\right)
\]
for each class \(c \in \mathcal{C}\).

Assuming that $\epsilon \ll 1$, then $\epsilon \ll \sqrt{\epsilon}$. Therefore, then worst case bound when $\epsilon \geq \epsilon_2'+\frac{\gamma'}{1+\gamma'}\epsilon_3'$ is achieved when $\epsilon_2'=\epsilon$:
$$\mathit{intra}_c\geq 1-8(\frac{C-1}{C\alpha\beta})\sqrt{\frac{2\epsilon(1-\delta)}{\delta\lambda}}$$
Plug in $\lambda=\exp(-O(C\alpha\beta))$ and with simplification we get:
$$\mathit{intra}_c\geq 1-\frac{(C-1)}{C\alpha\beta}\sqrt{\exp(O(C\alpha\beta))\frac{128\epsilon(1-\delta)}{\delta}}=1-O(\frac{e^{O(C\alpha\beta)}}{\alpha\beta}\sqrt{\frac{\epsilon}{\delta}})$$

Now consider the inter-class cosine similarity. Let $m_c=-\frac{C}{C-1}\dot{\mathbf{w}_c}\tilde{\mathbf{h}}_{c}$, by Lemma \ref{explemma} we know that for any set $S$ of $\delta (C-1)$ classes in $[C]-\{c\}$, using the definition that $\dot{\mathbf{w}}_c=\mathbf{w}_c-\tilde{\mathbf{w}}$ there is
$$\sum_{c'\in S}(\dot{\mathbf{w}}_{c'}-\dot{\mathbf{w}}_{c})\tilde{\mathbf{h}}_{c}=\sum_{c'\in S}(\mathbf{w}_{c'}-\mathbf{w}_{c})\tilde{\mathbf{h}}_{c}\leq \delta (C-1) m_c+(C-1)\sqrt{\frac{2\delta\epsilon_{1,c}'}{\exp(m_c)}}$$

Therefore, for at least $(1-\delta)(C-1)$ classes, there is 
\begin{align}
    \label{lowinterdot}(\dot{\mathbf{w}}_{c'}-\dot{\mathbf{w}}_{c})\tilde{\mathbf{h}}_{c}&\leq m_c+\sqrt{\frac{2\epsilon_{1,c}'}{\exp(m_c)\delta}}=-\frac{C}{C-1}\dot{\mathbf{w}_c}\tilde{\mathbf{h}}_{c}+\sqrt{\frac{2\epsilon_{1,c}'}{\exp(m_c)\delta}}\\
    \dot{\mathbf{w}}_{c'}\tilde{\mathbf{h}}_{c}&\leq -\frac{1}{C-1}\dot{\mathbf{w}_c}\tilde{\mathbf{h}}_{c}+\sqrt{\frac{2\epsilon_{1,c}'}{\exp(m_c)\delta}}
\end{align}
Combining with \eqref{largewh} \eqref{smallab}, we get that there are at least $(1-2\delta)C\times (1-3\delta)C\geq (1-5\delta)C^2$ pairs of classes $c,c'$ that satisfies the following: for both $c$ and $c'$, equations \eqref{largewh} \eqref{smallab} are not satisfied (i.e. satisfied in reverse direction), and \eqref{lowinterdot} is satisfied for the pair $c', c$. Note that this implies $$m_c=-\frac{C}{C-1}\dot{\mathbf{w}}_c\mathbf{\tilde{h}}_c\leq -\frac{C}{C-1}\alpha\beta+\epsilon_3'+\sqrt{\frac{2\epsilon_2'(1-\delta)}{\delta\lambda}}$$ and $$\dot{\mathbf{w}}_{c'}\tilde{\mathbf{h}}_{c}\leq -\frac{\alpha\beta}{C-1}+\frac{1}{C}(\epsilon_3'+\sqrt{\frac{2\epsilon_2'(1-\delta)}{\delta\lambda}})+\sqrt{\frac{2\epsilon_{1,c}'}{\exp(m_c)\delta}}$$
We now seek to simplify the above bounds using the constraint that $\epsilon \geq \frac{1}{C}\sum_{c=1}^C\frac{\epsilon_{1,c}'}{C}+\epsilon_2'+\frac{\gamma'}{1+\gamma'}\epsilon_3'$. Note that $\epsilon \ll \sqrt{\epsilon}$, and both $\lambda$ and $\exp(m_c)$ are $\exp(-O(C\alpha\beta))$, therefore, we can achieve the maximum bound by setting $\epsilon_{1,c}'=\epsilon$,
$$\dot{\mathbf{w}}_{c'}\tilde{\mathbf{h}}_{c}\leq -\frac{\alpha\beta}{C-1}+\exp(O(C\alpha\beta))\sqrt{\frac{2\epsilon}{\delta}}$$
Similarly, we can achieve the smallest bound on $\alpha_c\beta_c$ (the reverse of \eqref{smallab})by setting $\epsilon_2'=\epsilon$ and using $\lambda=\exp(-O(\alpha\beta))$ we get for both $c$ and $c'$
$$\alpha_c\beta_c\leq \alpha\beta+\exp(O(\alpha\beta))\sqrt{\frac{2\epsilon}{\delta}}$$
and achieve the largest bound on $\dot{\mathbf{w}}_c\mathbf{\tilde{h}}_c$ (the reverse of \eqref{largewh}) by setting $\epsilon_2'=\epsilon$ we get for both $c$ and $c'$:
$$\dot{\mathbf{w}}_c\mathbf{\tilde{h}}_c \leq \alpha\beta-\exp(O(C\alpha\beta))\sqrt{\frac{2\epsilon}{\delta}}$$
Therefore, we can apply Lemma \ref{interdivide} with $\alpha=\alpha_c$, $\beta=\beta_c$, $\epsilon'=\alpha_c\beta_c-\dot{\mathbf{w}}_c\mathbf{\tilde{h}}_c\leq 2\exp(O(C\alpha\beta))\sqrt{\frac{2\epsilon}{\delta}}$ bound to get:
\begin{align*}
    \cos_\angle(\dot{\mathbf{w}}_{c'}, \tilde{\bar{\mathbf{h}}}_{c})&\leq -\frac{1}{C-1}+\frac{C}{C-1}\frac{\exp(O(C\alpha\beta))}{\alpha\beta}\sqrt{\frac{2\epsilon}{\delta}}+4(\frac{2\exp(O(C\alpha\beta))}{\alpha\beta}\sqrt{\frac{2\epsilon}{\delta}})^{1/3}\\
    &\leq -\frac{1}{C-1}+O(\frac{e^{O(C\alpha\beta)}}{\alpha\beta}(\frac{\epsilon}{\delta})^{1/6}) 
\end{align*}
Where the last inequality is because $\frac{e^{O(C\alpha\beta)}}{\alpha\beta}>1, \frac{\epsilon}{\delta}<1$.
Finally, we derive an upper bound on $\cos_\angle(\tilde{\bar{\mathbf{h}}}_{c'}, \tilde{\bar{\mathbf{h}}}_{c})$ and thus intra-class cosine similarity by combining the above bounds. Note that for $\frac{\pi}{2}<a<\pi$ and $0<b<\frac{pi}{2}$ we have:
\begin{align*}
    \cos(a-b)&=\cos(a)\cos(b)+\sin(a)\sin(b)\\
    &\leq \cos(a)+\sin(b)\\
    &\leq \cos(a)+\sqrt{1-\cos^2(b)}\\
    &\leq cos(a)+\sqrt{2(1-\cos(b))}
\end{align*}
by \eqref{divbound} we get that
$$\cos_\angle(\dot{\mathbf{w}}_{c'}, \tilde{\bar{\mathbf{h}}}_{c'})\geq 1-(\frac{C-1}{C\alpha\beta})(\frac{\epsilon'_3}{\delta}-2\sqrt{\frac{2\epsilon_2'(1-\delta)}{\delta\lambda}})\geq 1-\frac{\exp(O(C\alpha\beta))}{\alpha\beta}\sqrt{\frac{2\epsilon}{\delta}}$$
Therefore, 
\begin{align*}
\cos_\angle(\tilde{\bar{\mathbf{h}}}_{c'}, \tilde{\bar{\mathbf{h}}}_{c})
&\leq \cos_\angle(\dot{\mathbf{w}}_{c'}, \tilde{\bar{\mathbf{h}}}_{c}) +\sqrt{2(1-\cos_\angle(\dot{\mathbf{w}}_{c'}, \tilde{\bar{\mathbf{h}}}_{c'}))}\\
&\leq -\frac{1}{C-1}+\frac{C}{C-1}\frac{\exp(O(C\alpha\beta))}{\alpha\beta}\sqrt{\frac{2\epsilon}{\delta}}+4(\frac{2\exp(O(C\alpha\beta))}{\alpha\beta}\sqrt{\frac{2\epsilon}{\delta}})^{1/3}+\sqrt{\frac{\exp(O(C\alpha\beta))}{\alpha\beta}\sqrt{\frac{2\epsilon}{\delta}}}\\
&= -\frac{1}{C-1} +O(\frac{e^{O(C\alpha\beta)}}{\alpha\beta}(\frac{\epsilon}{\delta})^{1/6})
\end{align*}
Since $\|\tilde{\bar{\mathbf{h}}}_{c}\|\leq 1$, there is
\begin{align*}
\tilde{\bar{\mathbf{h}}}_{c'}\cdot \tilde{\bar{\mathbf{h}}}_{c}=\|\tilde{\bar{\mathbf{h}}}_{c'}\| \|\tilde{\bar{\mathbf{h}}}_{c}\|\cos_\angle(\tilde{\bar{\mathbf{h}}}_{c'}, \tilde{\bar{\mathbf{h}}}_{c})\leq -\frac{1}{C-1} +O(\frac{e^{O(C\alpha\beta)}}{\alpha\beta}(\frac{\epsilon}{\delta})^{1/6})
\end{align*}
Applying \ref{avgnormvec} shows the bound on inter-class cosine similarity. Note that although this bound holds only for $1-5\delta$ fraction of pairs of classes, changing the fraction to $1-\delta$ only changes $\delta$ by a constant factor and does not affect the asymptotic bound.
\end{proof}
\subsection{Proof of Theorem \ref{thm:main_cor}}
 \begin{theorem}[Detailed Version of \ref{thm:main_cor}]
    \label{maincor2_appendix}
    For an neural network classifier without bias terms trained on a dataset with the number of classes $C\geq 3$ and samples per class $N\geq 1$, under the following assumptions:
    \begin{enumerate}
        \item The network contains an batch normalization layer without bias term before the final layer with trainable weight vector $\boldsymbol{\gamma}$;
        \item The layer-peeled regularized cross-entropy loss with weight decay $\lambda<\frac{1}{\sqrt{C}}$ $$\mathcal{L}_{\mathrm{reg}}=\frac{1}{CN}\sum_{c=1}^C \sum_{i=1}^{N} \mathcal{L}_{\mathrm{CE}}\left(f(\boldsymbol{x}_{c,i};\boldsymbol{\theta}), \boldsymbol{y}_c\right)+\frac{\lambda}{2}(\|\boldsymbol{\gamma}\|^2+\|\mathbf{W}\|_F^2)$$ satisfies $\mathcal{L}_{\mathrm{reg}}\leq m_{\mathrm{reg}}+\epsilon$ for small $\epsilon$;
    where $m_{reg}$ is the minimum achievable regularized loss
    \end{enumerate}
    then for at least $1-\delta$ fraction of all classes , with $\frac{\epsilon}{\delta}\ll1$, $\epsilon < \lambda$ and for small constant $\kappa>0$ and $\rho=(\frac{Ce}{\lambda})^{\kappa C}$ there is
    $$\mathit{intra}_c\geq 1-\frac{C-1}{C}\sqrt{\frac{128\rho\epsilon(1-\delta)}{\delta}}=1-O\left(\left(\frac{C}{\lambda}\right)^{O(C)}\sqrt{\frac{\epsilon}{\delta}}\right),$$
    and also for a cosine similarity representation of NC3 in \cite{doi:10.1073/pnas.2015509117}:
    $$\cos_{\angle}(\dot{\mathbf{w}}_c, \tilde{\mathbf{h}_c})\geq 1-2\sqrt{\frac{2\rho\epsilon(1-\delta)}{\delta}}=1-O\left(\left(\frac{C}{\lambda}\right)^{O(C)}\sqrt{\frac{\epsilon}{\delta}}\right),$$
    and for at least $1-\delta$ fraction of all pairs of classes $c,c'$, with $\frac{\epsilon}{\delta}\ll1$, there is
    \begin{align*}
    inter_{c,c'}\leq -\frac{1}{C-1}+\frac{C\rho}{C-1}\sqrt{\frac{2\epsilon}{\delta}}+4(\rho\sqrt{\frac{2\epsilon}{\delta}})^{1/3}+\sqrt{\rho\sqrt{\frac{2\epsilon}{\delta}}}
 = -\frac{1}{C-1} +O(\left(\frac{C}{\lambda}\right)^{O(C)}(\frac{\epsilon}{\delta})^{1/6})
\end{align*}
\end{theorem}
\begin{proof}
    Let $\boldsymbol{\gamma}^*$ and $\boldsymbol{W}^*$ be the weight vector and weight matrix that achieves the minimum achievable regularized loss. Let $\alpha=\|\boldsymbol{\gamma}\|$ and $\beta=\frac{\|\boldsymbol{W}\|_F}{\sqrt{C}}$, and $\alpha^*$ and $\beta^*$ represent the values at minimum loss accordingly. According to Lemma \ref{bnprop}, we know that $\sqrt{\frac{1}{N}\sum_{i=1}^N \|\mathbf{h}_i\|_2^2}=\|\boldsymbol \gamma\|_2=\alpha$. From Theorem \ref{main2} we know that, under fixed $\alpha\beta$, the minimum achievable unregularized loss is $\log(1+(C-1)\exp(-\frac{C}{C-1}\alpha\beta))$. Since only the product $\gamma=\alpha\beta$ is of interest to Theorem \ref{main2}, we make the following observation:
    \begin{align*}
    \mathcal{L}_{\mathrm{reg}}&=\frac{1}{CN}\sum_{c=1}^C \sum_{i=1}^{N} \mathcal{L}_{\mathrm{CE}}\left(f(\boldsymbol{x}_{c,i};\boldsymbol{\theta}), \boldsymbol{y}_c\right)+\frac{\lambda}{2}(\|\boldsymbol{\gamma}\|^2+\|\mathbf{W}\|_F^2)\\
    &\geq \log(1+(C-1)\exp(-\frac{C}{C-1}\alpha\beta))+\frac{\lambda}{2}(\alpha^2+C\beta^2)\\
    &\geq \log(1+(C-1)\exp(-\frac{C}{C-1}\gamma))+\sqrt{C}\lambda\gamma\\
    &\geq \min_{\gamma} \log(1+(C-1)\exp(-\frac{C}{C-1}\gamma))+\sqrt{C}\lambda\gamma
    \end{align*}
Now we analyze the properties of this function. For simplicity, we combine $\sqrt{C}\lambda$ into $\lambda$ in the following proposition:
\begin{proposition}
\label{funcprop}
The function $f_\lambda(\gamma)=\log\left(1+(C-1)\exp(-\frac{C}{C-1}\gamma)\right)+\lambda\gamma$ have minimum value $$f_\lambda(\gamma^*)=\log(1-\frac{C-1}{C}\lambda)+\frac{C-1}{C}\lambda\log\left(\frac{C-(C-1)\lambda}{\lambda}\right)$$ achieved at $\gamma^*=O(\log(\frac{1}{\lambda}))$ for $\lambda<1$. Furthermore, for any $\gamma$ such that $f_\lambda(\gamma)-f_\lambda(\gamma^*)\leq \epsilon \ll \lambda$, there is $|\gamma-\gamma^*|\leq \sqrt{O(1/\lambda)\epsilon}$
\end{proposition} 
\begin{proof}
    
Consider the optimum of the function by setting the derivative to 0:
\begin{align*}
    g_\lambda'(\gamma^*)&= -\frac{C}{C-1}\frac{(C-1)\exp(-\frac{C}{C-1}\gamma^*)}{\big(1+(C-1)\exp(-\frac{C}{C-1}\gamma^*)\big)} + \lambda=0\\
    \frac{C-1}{C}\lambda&=1-\frac{1}{1+(C-1)\exp(-\frac{C}{C-1}\gamma^*)}\\
    1+(C-1)\exp(-\frac{C}{C-1}\gamma^*)&=\frac{1}{1-\frac{C-1}{C}\lambda}\\
    \gamma^*&=\frac{C-1}{C}\log\left(\frac{C-(C-1)\lambda}{\lambda}\right)<\log(\frac{C}{\lambda})
\end{align*}
Plugging in $\gamma^*=\frac{C-1}{C}\log\left(\frac{C-(C-1)\lambda}{\lambda}\right)$ to the original formula we get:
$$f_\lambda(\gamma^*)=\log(1-\frac{C-1}{C}\lambda)+\frac{C-1}{C}\lambda\log\left(\frac{C-(C-1)\lambda}{\lambda}\right)$$
Note that since $\gamma\geq 0$, the optimum point is only positive when $\lambda\leq 1$.

Now consider the case where the loss is near-optimal and $\gamma=\gamma^*+\epsilon'$ for $\epsilon' \ll 1$:
\begin{align*}
    &\log\left(1+(C-1)\exp(-\frac{C}{C-1}(\gamma^*+\epsilon'))\right)+\lambda (\gamma^*+\epsilon')\\
    \geq &\log\left(1+(C-1)\exp(-\frac{C}{C-1}\gamma^*)(1-\frac{C}{C-1}\epsilon'+\frac{\epsilon'^2}{2})\right)+\lambda (\gamma^*+\epsilon')\\
    \geq &\log\left(1+(C-1)\exp(-\frac{C}{C-1}\gamma^*)\right)+\frac{(C-1)\exp(-\frac{C}{C-1}\gamma^*)}{\big(1+(C-1)\exp(-\frac{C}{C-1}\gamma^*)\big)}(-\frac{C}{C-1}\epsilon'+\frac{\epsilon^2}{2})+\lambda (\gamma^*+\epsilon')
\end{align*}
By definition of $\gamma^*$ as the optimal $\gamma$, the first-order term w.r.t. $\epsilon'$ must cancel out. Also, by plugging in $\gamma^*$, the coefficient of $\frac{\epsilon'^2}{2}$ is $\frac{C-1}{C}\gamma$. Therefore,
\begin{align*}
    &\log\left(1+(C-1)\exp(-\frac{C}{C-1}(\gamma^*+\epsilon'))\right)+\lambda (\gamma^*+\epsilon')\\
    \leq &\log\left(1+(C-1)\exp(-\frac{C}{C-1}\gamma^*)\right)+\lambda\gamma^*+\frac{C-1}{C}\lambda\epsilon'^2
\end{align*}
Conversely, for any $\epsilon\ll 1$ for which $g(\gamma)\leq g(\gamma^*)+\epsilon$, there must be $|\gamma-\gamma^*|\leq \sqrt{\frac{C\epsilon}{(C-1)\lambda}}$ 
\end{proof}
Thus, the minimum achievable value of the regularized loss is
$$m_{\mathrm{reg}}=\log(1-\frac{C-1}{\sqrt{C}}\lambda)+\frac{C-1}{\sqrt{C}}\lambda\log\left(\frac{\sqrt{C}}{\lambda}-(C-1)\right)$$
Now, consider any $\mathbf{W}$ and $\boldsymbol{\gamma}$ that achieves near-optimal regularized loss $\mathcal{L}_{\mathrm{reg}}=m_{\mathrm{reg}}+\epsilon$ for very small $\epsilon$. Recall that $\alpha=\|\bg\|$, $\beta=\frac{\|\mathbf{W}\|_F}{\sqrt{C}}$, $\gamma=\alpha\beta$. According to Proposition \ref{funcprop} we know that $|\gamma-\gamma^*|\leq \sqrt{\frac{C\epsilon}{(C-1)\lambda}}$. Therefore, $\gamma\leq \gamma^*+\sqrt{\frac{C\epsilon}{(C-1)\lambda}}=O(\log(C/\lambda))+\sqrt{\frac{C\epsilon}{(C-1)\lambda}}$. Also, note that $\mathcal{L}_{\mathrm{reg}}-f_{\sqrt{C}\lambda}(\gamma)\leq\mathcal{L}_{\mathrm{reg}}-f_{\sqrt{C}\lambda}(\gamma^*)=\epsilon$, where $f_{\sqrt{C}\lambda}(\gamma)$ is the minimum unregularized loss according to Theorem \ref{main2}. Therefore, we can apply Theorem \ref{main2} with $\alpha\beta=\gamma<O(\log(C/\lambda))+\sqrt{\frac{C\epsilon}{(C-1)\lambda}}$ and the same $\epsilon$ to get the results in the theorem.

\end{proof}

\end{document}